\documentclass[11pt]{article}
\usepackage{amsmath,amssymb,amsthm}
\usepackage{graphicx}
\usepackage[numbers]{natbib}
\usepackage[margin=1in, top=0.2in, headheight=10pt, headsep=5pt, includehead]{geometry}
\usepackage{xcolor}
\usepackage[colorlinks=true, citecolor=blue, linkcolor=red]{hyperref}
\usepackage{titlesec}

\definecolor{tocgreen}{RGB}{34, 139, 34}

\makeatletter
\AtBeginDocument{%
  \let\orig@starttoc\@starttoc
  \renewcommand{\@starttoc}[1]{%
    \begingroup
    \hypersetup{linkcolor=tocgreen}%
    \orig@starttoc{#1}%
    \endgroup
  }%
}
\makeatother
\usepackage[T1]{fontenc}
\usepackage{newtxtext, newtxmath}  
\usepackage{tocloft}              

\numberwithin{equation}{section}
\titleformat{\section}
  {\normalfont\Large\bfseries\MakeUppercase}{\thesection}{1em}{| }

\titleformat{\subsection}
  {\normalfont\large\bfseries}{\thesubsection}{1em}{| }

\titleformat{name=\section,numberless}
  {\normalfont\Large\bfseries}{}{0pt}{}

\makeatletter
\newtheoremstyle{dotstyle}{3pt}{3pt}{\normalfont}{}{\bfseries}{.}{ }{\thmname{#1}\thmnumber{ #2}\thmnote{ (#3)}}
\theoremstyle{dotstyle}
\newtheorem{theorem}{Theorem}[section]

\newtheorem{proposition}{Proposition}[section]
\newtheorem{definition}{Definition}[section]
\newtheorem{remark}{Remark}[section]
\newtheorem{corollary}{Corollary}[section]
 
\renewenvironment{proof}[1][\proofname]{%
  \par\pushQED{\qed}%
  \topsep6\p@\@plus6\p@\relax
  \trivlist
  \item[\hskip\labelsep\normalfont\bfseries #1\@addpunct{.}]%
  \ignorespaces
}{%
  \popQED\endtrivlist\@endpefalse
}
\makeatother

\usepackage{etoolbox}   
\appto\appendix{
  \titleformat{\section}                         
    {\normalfont\Large\bfseries\MakeUppercase}  
    {APPENDIX~\thesection:}{1em}{}              
  \titleformat{name=\section,numberless}         
    {\normalfont\Large\bfseries}{}{0pt}{}
}

\title{\bfseries Memory Prediction Excess: A Probabilistic Quantity for Predictive Gain and Memory Length in Stochastic Processes}
\author{
  Jiahao Jiang\thanks{Correspondence to: Jiahao Jiang, School of Mathematics, Southwest Jiaotong University. Email: \href{mailto:jiahao.jiang@swjtu.edu.cn}{jiahao.jiang@swjtu.edu.cn}} \\
  \textit{School of Mathematics, Southwest Jiaotong University, Chengdu 610031, China}
}
\date{}

\begin{document}

\maketitle
\begin{center}
\large\bfseries Abstract
\end{center}
\vspace{10pt}
A central question in the prediction of stochastic processes is the extent to which past information can improve the probability of correctly predicting the next state. To address this question in quantitative terms, we introduce the Memory Prediction Excess (MPE), which measures the average improvement in prediction accuracy obtained by using the entire observed history relative to using only the static marginal distribution, in the setting of discrete-time finite-state processes. More formally, the MPE is defined as the difference between the expected optimal conditional prediction accuracy and the optimal static prediction accuracy. Its basic properties are examined: the MPE is always non-negative; it admits an upper bound that depends on the static accuracy, and this bound is attained if and only if the future is almost surely a deterministic function of the past; degenerate cases in which the MPE vanishes are also characterized. A normalized version, taking values in the unit interval, is introduced as a dimensionless measure of predictive efficiency. A lower bound for the MPE is derived by comparing predictions based on histories of different lengths, showing that the expected optimal prediction accuracy is monotone with respect to the history length.

The framework is further extended to finite-length histories, where the finite-history MPE (FH-MPE) is formulated to measure the predictive gain attainable when only the most recent observations are retained. This extension allows us to explore how the predictive gain depends on the length of the available history and leads to the notion of a minimal memory length required to achieve the same predictive performance as the full history. As an illustration of the broader applicability of these ideas, we turn to the important class of finite-order Markov chains, where this minimal memory length is shown to be bounded by the Markov order. The MPE and its variants are formulated in terms of conditional probabilities and prediction accuracies, offering a probabilistic perspective on the predictive utility of memory that is complementary to classical information-theoretic approaches. More broadly, the investigation touches on the structure of prediction probabilities themselves—looking at how optimal prediction accuracy evolves with available information, and how conditional and marginal distributions are related through the lens of prediction accuracy—with a view toward contributing to a probabilistic understanding of the relationship between memory and prediction in stochastic processes. It is hoped that this perspective may also prove beneficial in settings where memory and prediction are closely intertwined, such as sequential decision-making, adaptive learning, and the modeling of temporally dependent data.

\vspace{10pt}

\vspace{6pt}

\vspace{10pt}  
{\centering\Large\bfseries\MakeUppercase{Contents}\par}
\vspace{6pt}
\makeatletter
\@starttoc{toc}
\makeatother
\vspace{12pt}  

\section{INTRODUCTION}

The prediction of future states of a stochastic process based on its observed past is a fundamental problem in probability theory and statistics, with broad applications across the physical, engineering, and social sciences. More broadly, the understanding, modeling, and prediction of the evolution of stochastic systems is among the central themes in much of probability theory and its applications. Within this landscape, a recurring question concerns how information from the past shapes our ability to understand and forecast future behavior—an issue that has been investigated from various perspectives, leading to an extensive collection of quantitative measures and modeling frameworks. This line of inquiry manifests in several directions, each offering a distinct lens through which the relationship between past and future can be viewed.

A broad approach to quantifying the relationship between past and future has been developed across several mathematical disciplines, from information theory to functional analysis. The foundational work \cite{shannon1948mathematic} addressed the fundamental problem of communication as that of reproducing at one point either exactly or approximately a message selected at another point. A key insight was that the significant aspect of communication is that the actual message is one selected from a set of possible messages, and the system must be designed to operate for each possible selection. This perspective naturally leads to the quantification of information and uncertainty, forming the basis for later developments in information theory. Complementing this, a framework \cite{renyi1959measures} formulated a set of postulates for measures of dependence between random variables and studied maximal correlation as a quantity satisfying these postulates, focusing on the strength of statistical association. The notion of transfer entropy \cite{schreiber2000measuring} was introduced as an asymmetric measure that distinguishes information actually exchanged between interacting systems from shared information arising from common history or common input signals; it detects the direction of influence and asymmetry in the interaction of subsystems. More recently, the conditional and marginal mutual information in Gaussian and long-memory time series have been studied \cite{chavez2016conditional}, providing explicit characterizations of predictability in terms of sensitivity to random shocks and autocorrelation structure. In parallel, memory effects have also been studied within a functional-analytic setting. A two-parameter memory-weighted velocity operator has recently been constructed \cite{jiang2026foundations}, in which two independent continuous exponents separately weight past state increments and elapsed time scaling, thereby allowing different aspects of memory to evolve independently. A framework for nonlinear adaptive memory has also been proposed \cite{jiang2026functional}, in which memory-dependent functionals are constructed to capture how past events influence future states through hierarchical kernels and state-dependent sensitivity. These approaches may offer a further perspective on how the structure of memory can be modeled mathematically.

Alongside these information-theoretic developments, a substantial body of work has addressed prediction from the perspective of stochastic modeling and optimal estimation. A framework for consistent prediction rules based on marker processes in survival analysis \cite{jewell1993framework} provided a stochastic framework for prediction functions and demonstrated a consistency condition that such functions should satisfy. The study of Markov processes generated by generalized convolutions \cite{misiewicz2026markov} examines the converse relation between Markov processes and Lévy processes within generalized convolution frameworks. The generative nature of prediction has been explored \cite{lohr2009generative}, proposing an alternative notion of predictive models in terms of hidden Markov models capable of generating the underlying stochastic process. Optimality and complexity in quantum-state stochastic processes have been investigated \cite{venegas2023optimality}, identifying mechanisms underlying unpredictability and memory requirements for optimal prediction. Related work \cite{venegas2020measurement} examined measurement-induced randomness and structure in controlled qubit processes, showing that projective measurement induces high complexity in these processes, including the requirement of infinite memory for optimal prediction. The analysis of deep neural networks through the information plane was further advanced \cite{shwartz2017opening}, demonstrating the effectiveness of this visualization and revealing that most training epochs are spent on compression of the input into efficient representations. Conformal prediction techniques have been applied to functional data \cite{lei2015conformal}, providing prediction sets with guaranteed finite-sample behavior under no distributional assumptions. The problem of optimal stochastic modeling has been explored in the context of unitary quantum dynamics \cite{liu2019optimal}, revealing advantages in memory reduction compared to classical models.

In a different but related direction, extensive research has focused on stochastic process modeling and prediction in various applied contexts. The analysis of stationary processes has been investigated from the perspective of spectral representations and dynamical systems \cite{didenko2014class}. The design and analysis of computer experiments when the output is highly correlated over the input space has been studied \cite{lim2002design}, showing that limiting best linear unbiased predictors involve Lagrange interpolating polynomials, with optimal designs developed based on minimizing the limiting integrated mean squared error of prediction. The calibration and prediction of computer models using Gaussian stochastic processes have been developed \cite{gu2018scaled}, proposing a scaled Gaussian process that bridges existing calibration methods. The statistical prediction of continuous-time autoregressive processes has been studied via reproducing kernel spaces \cite{mokhtari2002prediction}, constructing approximations of optimal predictors. Optimal transformations for prediction in continuous-time processes with finite past and future intervals have been investigated \cite{gidas2005optimal}, providing explicit constructions for rational spectral densities. A methodology relating weighted estimation criteria to integrated mean squared prediction error in stochastic processes has been introduced \cite{stehlik2020fredholm}. The optimal prediction of extremal quantities in L\'evy processes has been addressed \cite{bernyk2011predicting}, solving prediction problems for the ultimate supremum of stable processes. Trivariate stochastic models have been constructed for predicting maize yield \cite{chidzalo2022trivariate}, accounting for interdependence among climatic factors. The estimation of prediction mean squared error in Gaussian processes with exponential correlation structure has been examined \cite{abt1999estimating}, with emphasis on incorporating variability due to parameter estimation. The variance profile has been proposed as a characterization of stationary processes \cite{luati2012variance}, enabling direct derivation of prediction error formulas. Consistent prediction of well paths and geological surfaces has been addressed through smooth stochastic processes \cite{vazquez2024consistent}. A method based on stochastic process models for machine remaining useful life prediction has been developed \cite{lei2016new}, considering multiple variability sources of degradation processes. Self-data-driven methods for lifetime prediction of photovoltaic arrays have been proposed \cite{liu2023self}, accounting for uncertainty and volatility. Probabilistic remaining useful life prediction without lifetime labels has been achieved through a fusion of Bayesian deep learning and stochastic processes \cite{pan2024probabilistic}. Degradation modeling based on artificial neural network supported Tweedie exponential dispersion processes has been explored \cite{he2025degradation}. Functional variance processes have been used for degradation modeling and remaining useful lifetime prediction \cite{qin2024degradation}. Gaussian processes have been applied to flight delay prediction \cite{khanal2025gaussian}, learning stochastic structures from data. Extreme wind speed prediction has been improved through information transfer between correlated processes \cite{gaidai2019improving}. Scalable inversion algorithms for genomic prediction have been developed \cite{verbosio2017enhancing}, exploiting sparse and dense matrix structures.

A further line of investigation has explored the structure and complexity of stochastic processes from complementary angles. In particular, a systematic mathematical description of adaptive memory has been developed \cite{jiang2025towards}, in which the local regularity of a stochastic process adjusts dynamically in response to its own state; this leads from time-varying fractional Brownian motion to a class of responsive fractional Brownian motion with state-dependent Hurst exponents. Relatedly, empirical relationships between stock cross-correlation and volatility clustering have been studied \cite{micciche2013empirical}, revealing connections between predictability and interdependence. The role of correlations in the maximum distribution of strongly correlated Markovian processes has been numerically investigated \cite{micciche2025role}, with implications for convergence properties of time-average simulations. An operator theory approach to the evanescent part of two-parametric weak-stationary processes has been developed \cite{burdak2025operator}, characterizing dependence on the shape of the past. The role of generalization in using past experience to predict future behavior has been explored from a behavioral science perspective \cite{cowie2020generalizing}. Deep learning algorithms based on long short-term memory have been applied to seismic response uncertainty quantification \cite{kundu2022long}. Local limit theorems for random walks in intermittent dynamical environments have been established \cite{leppanen2026local}, extending central limit theory to inhomogeneous settings. A martingale approach to noncommutative stochastic calculus has been proposed \cite{jekel2025martingale}, providing a general theory of stochastic integration for noncommutative processes. Rigorous derivations of damped-driven wave turbulence theory have been obtained \cite{grande2026rigorous}, justifying kinetic regimes for nonlinear wave systems. Drift parameter estimation in fractional interacting particle systems has been addressed via Malliavin calculus \cite{amorino2025fractional}. Empirical forms of geometric inequalities have been established from a stochastic perspective \cite{paouris2026empirical}. A study of anomalous stochastic processes via generalizing fractional calculus has been presented \cite{jiang2025study}, advancing the existing formulation of fractional calculus for studying anomalous diffusion and stochastic analysis in fractional orders. Tests of equal forecast accuracy for factor-augmented regressions with weaker loadings have been developed \cite{margaritella2026new}, providing theoretical foundations for forecast comparison.

The literature on prediction in stochastic processes and related areas of probability theory is vast and multifaceted. Information-theoretic measures such as entropy, mutual information, transfer entropy, and their variants quantify information content and information flow; dependence measures such as maximal correlation capture the strength of statistical association; stochastic models, including Markov chains, Gaussian processes, and their generalizations, provide probabilistic descriptions of the future conditional on the past; and a range of prediction methodologies, from consistent prediction rules to conformal prediction and neural-network-based approaches, aim to construct optimal or approximately optimal predictors in various settings. Each of these lines of work offers a distinct perspective—whether focused on the amount of information available, the nature of the conditional distribution, the form of the optimal predictor, or the practical construction of prediction rules. Together, they form a rich background that both motivates and informs the present investigation. An additional direction arising from these developments concerns the predictive gain itself: by how much, in probabilistic terms, does the use of past information improve the probability of correctly predicting the next state? This direction centers on the improvement in accuracy that the information yields—a question that complements the above perspectives. The present work explores this direction by examining the gain in optimal prediction accuracy that past information can provide, along with how this gain behaves as the available history varies in length.

To make this notion precise, we introduce a quantity that captures the improvement just described. This quantity, which we term the \emph{Memory Prediction Excess} (MPE), is defined as the difference between the expected optimal conditional prediction accuracy and the optimal static prediction accuracy. It measures, in probabilistic terms, the average improvement in the probability of correctly predicting a future state when the entire observed history is used, relative to using only the static marginal distribution. The investigation focuses on the intrinsic probabilistic structure of optimal prediction, examining how the conditional distribution of the future given the past relates to its marginal counterpart through the lens of prediction accuracy. This perspective sits readily alongside broader information-theoretic approaches, offering a direct account of the predictive utility of memory at the level of probabilities themselves. Within the setting of discrete-time finite-state stochastic processes, the MPE is examined in some detail: its non-negativity, a sharp upper bound in terms of the static accuracy, the condition for attaining this bound, and the degenerate cases in which the MPE vanishes are characterized. A normalized version of the MPE is also introduced, which measures the proportion of the maximal possible predictive gain that is actually achieved, yielding a dimensionless index of predictive efficiency. A lower bound for the MPE is derived by comparing predictions based on histories of different lengths, showing that the expected optimal prediction accuracy is monotone with respect to the history length. The structure of the paper is as follows. In Section~\ref{sec:definition}, we introduce the notation for histories and conditional probabilities, define the static optimal prediction accuracy, and present the definition of the Memory Prediction Excess. In Section~\ref{sec:properties}, we establish the fundamental properties of the MPE, including non-negativity, the upper bound, its sharpness condition, and the degenerate cases. In Section~\ref{sec:normalized}, we define the normalized Memory Prediction Excess, characterize its extremal values, and derive a lower bound for the MPE. We also discuss the relation of our quantities to classical information-theoretic measures such as entropy, mutual information, and transfer entropy, highlighting the complementary perspectives they offer. In Section~\ref{sec:finite_history}, we extend the MPE framework to finite-length histories and establish the monotonicity of the finite-history MPE with respect to the memory length. For the important class of finite-order Markov chains, we show that the finite-history MPE saturates once the memory length reaches the Markov order. Finally, we introduce the minimal memory length required for the finite-history MPE to attain the full-history value, and examine its basic properties, including an explicit upper bound for Markov chains. Overall, these observations offer some insight into how prediction accuracy may improve with the use of history, how this improvement relates to the length of the available past, and related aspects explored in the sequel.

\section{DEFINITION OF MEMORY PREDICTION EXCESS}
\label{sec:definition}

This chapter introduces the central object of the present investigation: the Memory Prediction Excess (MPE). The MPE quantifies, in probabilistic terms, the average improvement in optimal prediction accuracy obtained by using the entire observed history, relative to a prediction based solely on the static marginal distribution. The definition is given in Section~\ref{sec:definition} after setting up the necessary notation for histories, static prediction accuracy, and conditional probabilities given the past. The properties of this quantity will be examined in detail in Section~\ref{sec:properties}, where we establish its non-negativity, upper bound, sharpness condition, and degenerate cases. A normalized version and its information-theoretic context are discussed in Section~\ref{sec:normalized}, while extensions to finite-length histories are taken up in Section~\ref{sec:finite_history}.

Throughout this section, we fix a discrete-time stochastic process $\{X_n\}_{n\ge 1}$ defined on a probability space $(\Omega, \mathcal{F}, \mathbb{P})$. The process takes values in a finite state space $\mathcal{A}$ with $1 \le |\mathcal{A}| < \infty$. For simplicity and to avoid recurring discussions of null events, we assume $\mathbb{P}(X_n = a) > 0$ for all $n$ and all $a \in \mathcal{A}$; this assumption guarantees that all conditional probabilities that appear are well-defined in their elementary sense. The case of vanishing probabilities introduces only trivial modifications that do not affect the almost-sure statements in our analysis.

\subsection{Notation for Histories and Elementary Probabilities}
\label{subsec:notation}

Before introducing the notions of static and conditional prediction accuracy, we set up the necessary notation for histories and elementary probabilities. This subsection provides the basic setup in which these quantities will be formulated.

For each integer $n \ge 1$, we denote by 
\[
X_1^n := (X_1, X_2, \dots, X_n)
\]
the \textbf{history} (or the $n$-past) of the process up to time $n$. Since $\mathcal{A}$ is finite, the random vector $X_1^n$ takes values in the Cartesian product 
\[
\mathcal{A}^n := \big\{ (x_1, \dots, x_n) \;\big|\; x_i \in \mathcal{A},\; i=1,\dots,n \big\},
\]
which is again a finite set. For a concrete realization $h = (x_1, \dots, x_n) \in \mathcal{A}^n$, we write
\begin{equation}
p_{\text{hist}}(h) := \mathbb{P}(X_1^n = h) \label{eq:p_hist_def}
\end{equation}
for the probability of observing this specific historical path. The collection $\{p_{\text{hist}}(h)\}_{h \in \mathcal{A}^n}$ forms a probability distribution on $\mathcal{A}^n$.

\subsection{Static Optimal Prediction Accuracy}
\label{subsec:static_accuracy}

With the notation for histories and the underlying probabilities in place (see Subsection~\ref{subsec:notation}), we now define the baseline prediction accuracy against which the MPE will be measured. This baseline, denoted by $U_n$, is the optimal prediction accuracy achievable when no historical information is used—that is, when the predictor relies solely on the static marginal distribution of $X_{n+1}$. The quantity $U_n$ will serve as a reference throughout the paper: it appears in the definition of the MPE (Subsection~\ref{subsec:mpe_definition}), in the non-negativity and upper bound results (Section~\ref{sec:properties}), and in the finite-history extension (Section~\ref{sec:finite_history}) where it plays the same baseline role.

When no historical information is available, the most natural prediction for $X_{n+1}$ is based solely on its marginal distribution. For each $a \in \mathcal{A}$ and each $n \ge 1$, define the marginal probability $p_{n+1}(a) := \mathbb{P}(X_{n+1} = a)$. The \textbf{static optimal prediction accuracy} for time $n+1$ is then given by the largest marginal probability:
\begin{equation}
U_n := \max_{a \in \mathcal{A}} p_{n+1}(a), \qquad n \ge 1.
\label{eq:U_n_def}
\end{equation}
The quantity $U_n$ is a deterministic constant (not a random variable). It represents the best possible correct-prediction probability one can achieve by always guessing a state that achieves the maximum in \eqref{eq:U_n_def}. If the maximizing state is not unique, any choice yields the same value $U_n$; for later use we may fix an arbitrary state
\begin{equation}
a_n^* \in \arg\max_{a \in \mathcal{A}} p_{n+1}(a)
\label{eq:astar_selection}
\end{equation}
as a representative. This selected state $a_n^*$ will serve as a convenient reference in the proof of non‑negativity (Theorem~\ref{thm:nonneg}).

For completeness, we extend the definition to $n=0$ by setting
\[
U_0 := \max_{a \in \mathcal{A}} \mathbb{P}(X_1=a).
\]
This is consistent with the convention that $X_1^0$ represents the empty history, as adopted in the definitions of $\Delta_0^{\text{mem}}$ and $M_0$ in Definition~\ref{def:MPE} (see \eqref{eq:MPE_def} and \eqref{eq:M_n_def}).

\subsection{Conditional Probabilities Given the History}
\label{subsec:conditional}

With the static baseline $U_n$ defined in Subsection~\ref{subsec:static_accuracy}, we now introduce the conditional distributions that underlie the history-dependent prediction accuracy. For each possible future state $a \in \mathcal{A}$, the conditional probability $\mathbb{P}(X_{n+1}=a \mid X_1^n)$ captures the information that the observed history carries about the next outcome. These conditional probabilities will be used in Subsection~\ref{subsec:history_dependent} to define the optimal conditional prediction accuracy $M_n$ as a random variable, and subsequently in Subsection~\ref{subsec:mpe_definition} to define the MPE itself. In the finite-history extension (Section~\ref{sec:finite_history}), analogous conditional probabilities given truncated histories will play a central role.

Let $\mathcal{G}_n := \sigma(X_1^n)$ be the $\sigma$-algebra generated by the history $X_1^n$. For each state $a \in \mathcal{A}$, the conditional probability of $\{X_{n+1}=a\}$ given $X_1^n$ is defined as the random variable
\[
q_a^{(n)} := \mathbb{P}(X_{n+1}=a \mid X_1^n),
\qquad a \in \mathcal{A},
\]
which is $\mathcal{G}_n$-measurable. For a sample point $\omega \in \Omega$, write $h = X_1^n(\omega)$; then $q_a^{(n)}(\omega) = \mathbb{P}(X_{n+1}=a \mid X_1^n = h)$. In the discrete setting, for any history $h \in \mathcal{A}^n$ with $p_{\text{hist}}(h) > 0$, the elementary conditional probability is given by
\[
\mathbb{P}(X_{n+1}=a \mid X_1^n = h) = \frac{\mathbb{P}(X_{n+1}=a,\; X_1^n = h)}{p_{\text{hist}}(h)}.
\]
For histories with $p_{\text{hist}}(h) = 0$, the value of the conditional probability may be assigned arbitrarily (for instance, set it equal to the marginal probability $p_{n+1}(a)$). Such histories occur on a null set and therefore do not affect any expectation or almost‑sure statement; any choice yields the same almost‑sure properties.

\subsection{History-Dependent Optimal Prediction Accuracy as a Random Variable}
\label{subsec:history_dependent}

Building on the conditional probabilities introduced in Subsection~\ref{subsec:conditional}, we now define the history-dependent analogue of the static optimal prediction accuracy. For each realization of the history, the maximal conditional probability $M_n$ is a random variable that represents the highest achievable probability of correctly predicting $X_{n+1}$ given the entire observed past. This random variable will serve as the conditional counterpart to $U_n$ in the definition of the MPE (Subsection~\ref{subsec:mpe_definition}). Its expectation, together with the upper bound result in Subsection~\ref{subsec:upper_bound}, plays a central role in characterizing the extremal behavior of the MPE. In the finite-history setting (Section~\ref{sec:finite_history}), an analogous family of random variables $M_{n,k}$ will be introduced to study the dependence of predictive accuracy on memory length.

Given a specific history $h$, the optimal prediction for $X_{n+1}$ using this information is to choose a state that maximizes the conditional distribution $\{\mathbb{P}(X_{n+1}=a \mid X_1^n = h)\}_{a \in \mathcal{A}}$. Define the corresponding maximal conditional probability:
\begin{equation}
M(h) := \max_{a \in \mathcal{A}} \mathbb{P}(X_{n+1}=a \mid X_1^n = h).
\label{eq:M_h}
\end{equation}
When the history is random, $M(h)$ becomes a random variable through its dependence on $X_1^n$. More formally, we set
\begin{equation}
M_n(\omega) := \max_{a \in \mathcal{A}} \mathbb{P}(X_{n+1}=a \mid X_1^n)(\omega)
= M\big( X_1^n(\omega) \big),
\qquad \omega \in \Omega.
\label{eq:M_n_def}
\end{equation}
Thus $M_n$ is a $\mathcal{G}_n$-measurable random variable. It represents the highest probability of correctly predicting $X_{n+1}$ that can be achieved after observing the entire past $X_1^n$. Because the conditional distribution sums to unity over $a \in \mathcal{A}$, we have the elementary bounds
\[
\frac{1}{|\mathcal{A}|} \le M_n(\omega) \le 1 \qquad \text{for almost every } \omega \in \Omega.
\]
The lower bound follows from the fact that for any finite collection of non‑negative numbers $\{x_a\}_{a \in \mathcal{A}}$ with $\sum_{a} x_a = 1$, we have $\max_a x_a \ge \frac{1}{|\mathcal{A}|} \sum_a x_a = 1/|\mathcal{A}|$; the upper bound is trivial because probabilities never exceed $1$.

\subsection{Formal Definition of the MPE}
\label{subsec:mpe_definition}

We now introduce the central concept of this work. With the static accuracy $U_n$ defined in Subsection~\ref{subsec:static_accuracy} and the conditional accuracy $M_n$ introduced in Subsection~\ref{subsec:history_dependent}, we are now in a position to define the Memory Prediction Excess. This quantity captures the average improvement in optimal prediction accuracy attributable to the use of the full history. Its fundamental properties will be examined in Section~\ref{sec:properties}, where we establish non-negativity, a sharp upper bound, the condition for attaining this bound, and degenerate cases in which it vanishes. A normalized version is introduced in Section~\ref{sec:normalized} to facilitate comparisons across processes and time indices, while extensions to finite-length histories are developed in Section~\ref{sec:finite_history}.

\begin{definition}[Memory Prediction Excess (MPE)]\label{def:MPE}
For a fixed time index $n \ge 1$, the \textbf{Memory Prediction Excess} of the process $\{X_k\}_{k\ge 1}$ at step $n$ is defined as
\begin{equation}
\Delta_n^{\text{mem}} := \mathbb{E}[M_n] - U_n.
\label{eq:MPE_def}
\end{equation}
For completeness, we set $\Delta_0^{\text{mem}} := 0$. This is consistent with interpreting $X_1^0$ as the empty history, so that the conditional probability $\mathbb{P}(X_1=a \mid X_1^0)$ reduces to the unconditional probability $\mathbb{P}(X_1=a)$. Consequently $M_0 := \max_{a\in\mathcal{A}} \mathbb{P}(X_1=a)$ is a constant, and $\mathbb{E}[M_0] = \max_{a\in\mathcal{A}} \mathbb{P}(X_1=a) =: U_0$, yielding $\Delta_0^{\text{mem}} = \mathbb{E}[M_0] - U_0 = 0$. In expanded form,
\[
\Delta_n^{\text{mem}} = \mathbb{E}\left[ \max_{a \in \mathcal{A}} \mathbb{P}(X_{n+1}=a \mid X_1^n) \right] \;-\; \max_{a \in \mathcal{A}} \mathbb{P}(X_{n+1}=a).
\]
\end{definition}

The definition captures the average excess in optimal prediction accuracy that is gained by incorporating the historical information $X_1^n$, compared to a prediction that relies only on the static marginal distribution of $X_{n+1}$. The term ``excess'' will be fully justified in the next section, where we prove $\Delta_n^{\text{mem}} \ge 0$ (Theorem \ref{thm:nonneg}).

\subsection{Explicit Expansion in a Discrete State Space}
\label{subsec:explicit_expansion}

With the MPE defined in Subsection~\ref{subsec:mpe_definition}, we now make the expectation involved in its definition more explicit by expanding it over the finite set of possible histories. This expansion reveals the mathematical structure of the MPE as the difference between two operations—maximization over states followed by averaging over histories, versus averaging followed by maximization. This structure serves as a unifying thread for the results that follow: the non-negativity and extremal properties of the MPE, its finite-history analogue, and the notion of effective memory length all arise from examining this basic distinction in different settings. The same type of expansion will also be used in the finite-history setting of Subsection~\ref{subsec:finite_definition} when expectations over history blocks are considered.

Since $X_1^n$ takes values in the finite set $\mathcal{A}^n$, we can expand the expectation $\mathbb{E}[M_n]$ as a weighted sum over all possible histories $h \in \mathcal{A}^n$. Using the distribution $p_{\text{hist}}(h)$ from \eqref{eq:p_hist_def} and the definition $M(h)$ from \eqref{eq:M_h}, we obtain
\begin{equation}
\mathbb{E}[M_n] = \sum_{h \in \mathcal{A}^n} p_{\text{hist}}(h) \, M(h)
= \sum_{h \in \mathcal{A}^n} p_{\text{hist}}(h) \left( \max_{a \in \mathcal{A}} \mathbb{P}(X_{n+1}=a \mid X_1^n = h) \right).
\label{eq:E_Mn_expansion}
\end{equation}
Similarly, the marginal probabilities of $X_{n+1}$ can be expressed by averaging the conditional probabilities over all histories:
\begin{equation}
\mathbb{P}(X_{n+1}=a) = \sum_{h \in \mathcal{A}^n} p_{\text{hist}}(h) \; \mathbb{P}(X_{n+1}=a \mid X_1^n = h),
\qquad a \in \mathcal{A}.
\label{eq:marginal_expansion}
\end{equation}
Taking the maximum over $a$ on both sides yields the static accuracy
\begin{equation}
U_n = \max_{a \in \mathcal{A}} \; \sum_{h \in \mathcal{A}^n} p_{\text{hist}}(h) \; \mathbb{P}(X_{n+1}=a \mid X_1^n = h).
\label{eq:U_n_expansion}
\end{equation}
Combining \eqref{eq:E_Mn_expansion} and \eqref{eq:U_n_expansion} we obtain an explicit representation of the Memory Prediction Excess:
\[
\Delta_n^{\text{mem}} 
= \sum_{h \in \mathcal{A}^n} p_{\text{hist}}(h) \left( \max_{a \in \mathcal{A}} \mathbb{P}(X_{n+1}=a \mid X_1^n = h) \right)
- \max_{a \in \mathcal{A}} \sum_{h \in \mathcal{A}^n} p_{\text{hist}}(h) \, \mathbb{P}(X_{n+1}=a \mid X_1^n = h).
\]

The expression above reveals the mathematical structure of $\Delta_n^{\text{mem}}$: it is the difference between two operations — first taking a maximum over $a$ and then averaging over $h$ (first term), versus first averaging over $h$ and then taking the maximum over $a$ (second term). This representation suggests why $\Delta_n^{\text{mem}} \ge 0$ holds: for each fixed $h$, the maximum over $a$ dominates each individual term, and averaging preserves the inequality. A rigorous proof is given in Theorem \ref{thm:nonneg} in the following section.

\section{BASIC PROPERTIES OF MEMORY PREDICTION EXCESS}
\label{sec:properties}

With the MPE defined in Section~\ref{sec:definition}, we now establish its fundamental properties. We prove that the Memory Prediction Excess is non-negative (Theorem~\ref{thm:nonneg}), derive an elementary upper bound (Proposition~\ref{prop:upper_bound}), characterize the precise conditions under which this bound is attained (Theorem~\ref{thm:sharpness}), and examine two degenerate scenarios (Corollaries~\ref{cor:independence} and~\ref{cor:static_degenerate}) where $\Delta_n^{\text{mem}}=0$. These results together help illustrate the extremal behavior of the MPE, and will serve as a foundation for the normalized version introduced in Section~\ref{sec:normalized} and the finite-history extension developed in Section~\ref{sec:finite_history}. Throughout this section we work with the same setting as in the previous section: $\{X_n\}_{n\ge 1}$ is a discrete-time stochastic process on $(\Omega,\mathcal{F},\mathbb{P})$ taking values in a finite state space $\mathcal{A}$, with $\mathbb{P}(X_n=a)>0$ for all $n$ and all $a\in\mathcal{A}$ to avoid repeated discussions of null events.

\subsection{Non-negativity of the MPE}
\label{subsec:nonneg}

We first show that the MPE is always non-negative, confirming the intuitive fact that conditioning on the full history cannot, on average, reduce the optimal prediction accuracy relative to the static baseline. This non-negativity follows from the definition of the MPE (Subsection~\ref{subsec:mpe_definition}), by comparing the conditional and static optimal predictors. The non-negativity property is the first of several extremal results in this section; it will be complemented by the upper bound in Subsection~\ref{subsec:upper_bound} and the degenerate cases in Subsection~\ref{subsec:degenerate}. It also underlies the normalized version introduced in Section~\ref{sec:normalized}, where the MPE is rescaled by its maximal possible value.

\begin{theorem}[Non-negativity of MPE]\label{thm:nonneg}
For every $n \ge 1$,
\begin{equation}
\Delta_n^{\text{mem}} \ge 0.
\end{equation}
\end{theorem}

\begin{proof}
The proof proceeds by comparing the conditional prediction accuracy $M_n$ with the static accuracy $U_n$ through a carefully chosen reference state. We divide the argument into five steps.

\paragraph{Step 1: Selection of a static optimal state.}
Since $\mathcal{A}$ is finite, the maximum in the definition of $U_n$ (see \eqref{eq:U_n_def}) is attained. Choose any state that achieves this maximum and denote it by $a_n^*$:
\begin{equation}
\mathbb{P}(X_{n+1}=a_n^*) = \max_{a \in \mathcal{A}} \mathbb{P}(X_{n+1}=a) = U_n.
\label{eq:astar_def}
\end{equation}
The state $a_n^*$ is fixed and deterministic; it does not depend on the sample point $\omega$. If the maximizer is not unique, we pick an arbitrary one — any choice is admissible for the argument that follows.

\paragraph{Step 2: Pointwise inequality between $M_n$ and the conditional probability of $a_n^*$.}
For any $\omega \in \Omega$, recall from \eqref{eq:M_n_def} that
\[
M_n(\omega) = \max_{a \in \mathcal{A}} \mathbb{P}(X_{n+1}=a \mid X_1^n)(\omega).
\]
Since $M_n(\omega)$ is the maximum of the finite set $\{\mathbb{P}(X_{n+1}=a \mid X_1^n)(\omega) : a \in \mathcal{A}\}$, it is at least as large as any element of that set. In particular, taking $a = a_n^*$ we obtain
\begin{equation}
M_n(\omega) \ge \mathbb{P}(X_{n+1}=a_n^* \mid X_1^n)(\omega), \qquad \forall \omega \in \Omega.
\label{eq:pointwise_ineq}
\end{equation}
This inequality holds for every sample point, not merely almost surely.

\paragraph{Step 3: Taking expectations preserves the inequality.}
The expectation operator $\mathbb{E}[\cdot]$ is monotone: if $Z_1(\omega) \ge Z_2(\omega)$ for all $\omega$ (or almost all), then $\mathbb{E}[Z_1] \ge \mathbb{E}[Z_2]$. Applying this to \eqref{eq:pointwise_ineq} yields
\begin{equation}
\mathbb{E}[M_n] \ge \mathbb{E}\bigl[ \mathbb{P}(X_{n+1}=a_n^* \mid X_1^n) \bigr].
\label{eq:expectation_ineq}
\end{equation}

\paragraph{Step 4: Evaluation of the right‑hand side.}
We now compute the expectation of the conditional probability. Since $X_1^n$ takes values in the finite set $\mathcal{A}^n$, we can expand the expectation as a weighted sum over all possible histories $h \in \mathcal{A}^n$. Using the distribution $p_{\text{hist}}(h) = \mathbb{P}(X_1^n = h)$ (see \eqref{eq:p_hist_def}), we obtain
\begin{align}
\mathbb{E}\bigl[ \mathbb{P}(X_{n+1}=a_n^* \mid X_1^n) \bigr]
&= \sum_{h \in \mathcal{A}^n} \mathbb{P}(X_1^n = h) \; \mathbb{P}(X_{n+1}=a_n^* \mid X_1^n = h) \notag \\
&= \sum_{h \in \mathcal{A}^n} \mathbb{P}(X_{n+1}=a_n^*,\; X_1^n = h) \notag \\
&= \mathbb{P}(X_{n+1}=a_n^*). \label{eq:law_total_prob}
\end{align}
The first equality expands the expectation over all possible histories; histories with zero probability contribute nothing and can be ignored. The second equality follows from the multiplication rule $\mathbb{P}(A \mid B) \mathbb{P}(B) = \mathbb{P}(A \cap B)$, applied with $A = \{X_{n+1}=a_n^*\}$ and $B = \{X_1^n = h\}$. The third equality uses the fact that the events $\{X_{n+1}=a_n^*,\; X_1^n = h\}$ for distinct $h \in \mathcal{A}^n$ are mutually disjoint, and their union over all $h$ is precisely the event $\{X_{n+1}=a_n^*\}$. Hence the sum of their probabilities equals $\mathbb{P}(X_{n+1}=a_n^*)$.

By the definition of $a_n^*$ in \eqref{eq:astar_def}, we have $\mathbb{P}(X_{n+1}=a_n^*) = U_n$. Substituting this into \eqref{eq:law_total_prob} gives
\begin{equation}
\mathbb{E}\bigl[ \mathbb{P}(X_{n+1}=a_n^* \mid X_1^n) \bigr] = U_n.
\label{eq:exp_cond_result}
\end{equation}

\paragraph{Step 5: Conclusion.}
Inserting \eqref{eq:exp_cond_result} into the inequality \eqref{eq:expectation_ineq} yields $\mathbb{E}[M_n] \ge U_n$. Recalling the definition $\Delta_n^{\text{mem}} = \mathbb{E}[M_n] - U_n$ from \eqref{eq:MPE_def}, we obtain
\[
\Delta_n^{\text{mem}} \ge 0,
\]
which completes the proof.
\end{proof}

\begin{remark}[On the nature of the inequality]\label{rem:role_astar}
Theorem \ref{thm:nonneg} shows that $\Delta_n^{\text{mem}} \ge 0$ holds for discrete-time finite-state stochastic processes. The inequality is understood in expectation: averaging the optimal conditional prediction accuracy over all histories produces a value that is at least the static accuracy. This average positivity reflects the sense in which historical information contributes, on the whole, to prediction. The proof via $a_n^*$ indicates that the excess can be traced to the exchange of averaging and maximization — a basic structural feature of such processes. In this view, $\Delta_n^{\text{mem}}$ offers one possible way to quantify the predictive value of memory, measured as the average improvement in optimal prediction accuracy when the past is taken into account.
\end{remark}

\subsection{Upper Bound and Sharpness Condition}
\label{subsec:upper_bound}

We now examine the upper range of the MPE. An upper bound follows from the boundedness of conditional probabilities. We then characterize the precise condition under which this bound is attained, showing that equality holds if and only if the future is almost surely a deterministic function of the past. Together with the non-negativity result of Subsection~\ref{subsec:nonneg}, these findings help present a coherent picture of the extremal behavior of the MPE. The sharpness condition also has implications for the normalized version introduced in Section~\ref{sec:normalized}, where the maximal possible gain corresponds to the deterministic case.

\begin{proposition}[Upper bound for MPE]\label{prop:upper_bound}
For every $n \ge 1$,
\begin{equation}
\Delta_n^{\text{mem}} \le 1 - U_n.
\label{eq:upper_bound}
\end{equation}
\end{proposition}

\begin{proof}
For any sample point $\omega \in \Omega$, recall from \eqref{eq:M_n_def} that
\[
M_n(\omega) = \max_{a \in \mathcal{A}} \mathbb{P}(X_{n+1}=a \mid X_1^n)(\omega).
\]
For each fixed $\omega$, the quantity $\mathbb{P}(X_{n+1}=a \mid X_1^n)(\omega)$ is a conditional probability, hence it takes values in the interval $[0,1]$ for every $a \in \mathcal{A}$. The maximum of a finite collection of numbers each lying in $[0,1]$ also lies in $[0,1]$. Consequently,
\begin{equation}
0 \le M_n(\omega) \le 1 \qquad \text{for every } \omega \in \Omega.
\label{eq:M_n_bounds_pointwise}
\end{equation}
In particular, the upper bound $M_n(\omega) \le 1$ holds for all $\omega \in \Omega$.

Taking expectations on both sides of the inequality $M_n(\omega) \le 1$ and using the monotonicity of the expectation operator $\mathbb{E}[\cdot]$ (which preserves inequalities almost everywhere, and here the inequality holds everywhere), we obtain
\begin{equation}
\mathbb{E}[M_n] \le \mathbb{E}[1] = 1.
\label{eq:E_Mn_ub}
\end{equation}

Now recall the definition of the Memory Prediction Excess from \eqref{eq:MPE_def}:
\[
\Delta_n^{\text{mem}} = \mathbb{E}[M_n] - U_n.
\]
Substituting the estimate $\mathbb{E}[M_n] \le 1$ from \eqref{eq:E_Mn_ub} into this definition yields
\[
\Delta_n^{\text{mem}} = \mathbb{E}[M_n] - U_n \le 1 - U_n.
\]
This completes the proof.
\end{proof}

The quantity $1-U_n$ provides an upper bound for $\Delta_n^{\text{mem}}$. It reflects the basic property that conditional probabilities lie in $[0,1]$.

\begin{remark}[Distribution-free upper bound]\label{rem:absolute_upper}
Since for any probability distribution on a finite set $\mathcal{A}$ we have
\[
\max_{a\in\mathcal{A}} \mathbb{P}(X_{n+1}=a) \ge \frac{1}{|\mathcal{A}|},
\]
it follows that
\[
1 - U_n \le 1 - \frac{1}{|\mathcal{A}|}.
\]
Combining this with Proposition~\ref{prop:upper_bound} yields
\[
\Delta_n^{\text{mem}} \le 1 - U_n \le 1 - \frac{1}{|\mathcal{A}|}.
\]
Thus $\Delta_n^{\text{mem}}$ is bounded above by a constant that depends only on the cardinality of the state space, not on the specific distribution of the process. In particular, this bound is universal across all discrete-time finite-state stochastic processes satisfying the assumptions of Section~\ref{sec:definition}.
\end{remark}

The following theorem gives a precise characterization of the case where $\Delta_n^{\text{mem}}$ reaches this bound.

\begin{theorem}[Sharpness condition]\label{thm:sharpness}
For a discrete-time finite-state stochastic process $\{X_n\}_{n\ge 1}$, the following statements are equivalent:
\begin{enumerate}
    \item $\Delta_n^{\text{mem}} = 1 - U_n$;
    \item There exists a function $f: \mathcal{A}^n \to \mathcal{A}$ such that $X_{n+1} = f(X_1^n)$ almost surely.
\end{enumerate}
In words, the Memory Prediction Excess reaches its upper bound if and only if the future state $X_{n+1}$ is almost surely a deterministic function of the past $X_1^n$.
\end{theorem}

\begin{proof}
We prove the two implications separately.

\paragraph{(i) $\Delta_n^{\text{mem}} = 1 - U_n$ $\Longrightarrow$ $X_{n+1}$ is a function of $X_1^n$ almost surely.}

Assume $\Delta_n^{\text{mem}} = 1 - U_n$. By the definition of $\Delta_n^{\text{mem}}$ in \eqref{eq:MPE_def}, we have $\mathbb{E}[M_n] - U_n = 1 - U_n$, which simplifies to
\begin{equation}
\mathbb{E}[M_n] = 1.
\label{eq:E_Mn_eq_1}
\end{equation}

From \eqref{eq:M_n_bounds_pointwise} we know that $M_n(\omega) \le 1$ for every $\omega \in \Omega$. Consider the nonnegative random variable $1 - M_n(\omega) \ge 0$. Its expectation is
\[
\mathbb{E}[1 - M_n] = 1 - \mathbb{E}[M_n] = 1 - 1 = 0.
\]
A nonnegative random variable with zero expectation must be zero almost surely (if it were positive on a set of positive measure, the expectation would be positive). Hence
\begin{equation}
M_n(\omega) = 1 \quad \text{for almost every } \omega \in \Omega.
\label{eq:Mn_eq_1_as}
\end{equation}

Fix a sample point $\omega$ for which $M_n(\omega) = 1$ holds, and let $h = X_1^n(\omega) \in \mathcal{A}^n$ be the observed history. From the definition of $M_n$ in \eqref{eq:M_n_def} and the condition $M_n(\omega) = 1$, we obtain
\[
\max_{a \in \mathcal{A}} \mathbb{P}(X_{n+1}=a \mid X_1^n = h) = 1.
\]

The collection $\{\mathbb{P}(X_{n+1}=a \mid X_1^n = h)\}_{a \in \mathcal{A}}$ forms a conditional probability distribution on the finite set $\mathcal{A}$. By definition, each term lies in $[0,1]$, and they satisfy $\sum_{a \in \mathcal{A}} \mathbb{P}(X_{n+1}=a \mid X_1^n = h) = 1$.

We claim that there exists a unique state $a_0 \in \mathcal{A}$ such that $\mathbb{P}(X_{n+1}=a_0 \mid X_1^n = h) = 1$, and for all $a \neq a_0$, $\mathbb{P}(X_{n+1}=a \mid X_1^n = h) = 0$. Indeed, since the maximum of the collection equals $1$, at least one state attains the value $1$. If two distinct states $a_1 \neq a_2$ both satisfied $\mathbb{P}(X_{n+1}=a_i \mid X_1^n = h) = 1$, then the sum over all $a$ would be at least $1 + 1 = 2$, contradicting the fact that the sum is $1$. Therefore the state achieving probability $1$ is unique. Denote this unique state by $f(h)$:
\begin{equation}
\mathbb{P}(X_{n+1} = f(h) \mid X_1^n = h) = 1,
\label{eq:cond_prob_one}
\end{equation}
and for all $a \neq f(h)$, $\mathbb{P}(X_{n+1}=a \mid X_1^n = h) = 0$.

From \eqref{eq:Mn_eq_1_as}, the set of histories for which this uniqueness holds has probability one. Let
\[
\mathcal{H}^* := \bigl\{ h \in \mathcal{A}^n : \max_{a \in \mathcal{A}} \mathbb{P}(X_{n+1}=a \mid X_1^n = h) = 1 \bigr\}.
\]
Then $\mathbb{P}(X_1^n \in \mathcal{H}^*) = 1$. For each $h \in \mathcal{H}^*$, we have defined $f(h)$ as the unique state satisfying \eqref{eq:cond_prob_one}. For histories $h \notin \mathcal{H}^*$ (which form a null set), we may define $f(h)$ arbitrarily, for instance $f(h) = a_0$ for some fixed $a_0 \in \mathcal{A}$; this arbitrary choice does not affect any almost-sure statement because these histories occur with probability zero.

We now verify that $X_{n+1} = f(X_1^n)$ holds almost surely. For any $h \in \mathcal{H}^*$ with $\mathbb{P}(X_1^n = h) > 0$, we compute
\begin{align}
\mathbb{P}\bigl( X_{n+1} = f(h),\; X_1^n = h \bigr)
&= \mathbb{P}(X_1^n = h) \cdot \mathbb{P}(X_{n+1}=f(h) \mid X_1^n = h) \notag \\
&= \mathbb{P}(X_1^n = h) \cdot 1 = \mathbb{P}(X_1^n = h).
\end{align}
Consequently,
\[
\mathbb{P}\bigl( X_{n+1} \neq f(h),\; X_1^n = h \bigr) = \mathbb{P}(X_1^n = h) - \mathbb{P}\bigl( X_{n+1} = f(h),\; X_1^n = h \bigr) = 0.
\]

We now aggregate these local equalities to obtain a global statement. The events $\{X_1^n = h\}$ for $h \in \mathcal{A}^n$ form a partition of $\Omega$ (up to a null set). Therefore,
\[
\mathbb{P}\bigl( X_{n+1} \neq f(X_1^n) \bigr)
= \sum_{h \in \mathcal{A}^n} \mathbb{P}\bigl( X_{n+1} \neq f(h),\; X_1^n = h \bigr).
\]
We examine the summands according to the type of history $h$:

\begin{itemize}
    \item If $h \in \mathcal{H}^*$ and $\mathbb{P}(X_1^n = h) > 0$, then as shown above,
    \[
    \mathbb{P}\bigl( X_{n+1} \neq f(h),\; X_1^n = h \bigr) = 0.
    \]
    \item If $h \in \mathcal{H}^*$ but $\mathbb{P}(X_1^n = h) = 0$, the term is zero because the factor $\mathbb{P}(X_1^n = h)$ multiplies a bounded probability.
    \item If $h \notin \mathcal{H}^*$, recall from the definition of $\mathcal{H}^*$ and the fact $\mathbb{P}(X_1^n \in \mathcal{H}^*) = 1$ that $\mathbb{P}(X_1^n = h) = 0$ for all such $h$. Hence each of these terms is also zero.
\end{itemize}

Thus every term in the sum is zero, yielding
\[
\mathbb{P}\bigl( X_{n+1} \neq f(X_1^n) \bigr) = 0.
\]
Therefore $X_{n+1} = f(X_1^n)$ holds almost surely. This completes the proof of the forward implication.

\paragraph{(ii) $X_{n+1} = f(X_1^n)$ almost surely $\Longrightarrow$ $\Delta_n^{\text{mem}} = 1 - U_n$.}

Suppose there exists a function $f: \mathcal{A}^n \to \mathcal{A}$ such that $X_{n+1} = f(X_1^n)$ holds almost surely. We need to show that under this assumption $\mathbb{E}[M_n] = 1$, from which $\Delta_n^{\text{mem}} = 1 - U_n$ follows immediately by the definition $\Delta_n^{\text{mem}} = \mathbb{E}[M_n] - U_n$.

\paragraph{Step 1: Conditional probabilities given a positive-probability history.}
Fix an arbitrary history $h \in \mathcal{A}^n$ with $\mathbb{P}(X_1^n = h) > 0$. 
By the assumption of this implication, we have $X_{n+1} = f(X_1^n)$ almost surely. 
Let 
\[
\mathfrak{N} := \bigl\{ \omega \in \Omega : X_{n+1}(\omega) \neq f(X_1^n(\omega)) \bigr\}
\]
be the null set on which the equality fails; thus $\mathbb{P}(\mathfrak{N}) = 0$.

Now consider the event $\{X_1^n = h\}$. On this event, $f(X_1^n) = f(h)$ is a constant 
(since $X_1^n(\omega) = h$ implies $f(X_1^n(\omega)) = f(h)$). 
Therefore, for any $\omega \in \{X_1^n = h\} \setminus \mathfrak{N}$, we have
\[
X_{n+1}(\omega) = f(X_1^n(\omega)) = f(h).
\]
In other words, on $\{X_1^n = h\}$, the equality $X_{n+1} = f(h)$ holds for every 
$\omega \in \{X_1^n = h\} \setminus \mathfrak{N}$.

Consequently,
\[
\bigl\{ X_{n+1} \neq f(h),\; X_1^n = h \bigr\} \subseteq \{X_1^n = h\} \cap \mathfrak{N} \subseteq \mathfrak{N}.
\]
Since $\mathfrak{N}$ is a null set ($\mathbb{P}(\mathfrak{N}) = 0$), the monotonicity of probability yields
\[
\mathbb{P}\bigl( X_{n+1} \neq f(h),\; X_1^n = h \bigr) \le \mathbb{P}(\mathfrak{N}) = 0.
\]
Hence
\begin{equation}
\mathbb{P}\bigl( X_{n+1} \neq f(h),\; X_1^n = h \bigr) = 0.
\label{eq:almost_equal_on_atom}
\end{equation}

Using the multiplication rule $\mathbb{P}(A \cap B) = \mathbb{P}(B) \cdot \mathbb{P}(A \mid B)$ with $B = \{X_1^n = h\}$, we obtain
\[
\mathbb{P}(X_1^n = h) \cdot \mathbb{P}\bigl( X_{n+1} \neq f(h) \mid X_1^n = h \bigr)
= \mathbb{P}\bigl( X_{n+1} \neq f(h),\; X_1^n = h \bigr) = 0.
\]
Since $\mathbb{P}(X_1^n = h) > 0$, we can divide by this positive quantity to conclude
\[
\mathbb{P}\bigl( X_{n+1} \neq f(h) \mid X_1^n = h \bigr) = 0.
\]

Because the conditional probability distribution on $\mathcal{A}$ sums to $1$, it follows that
\begin{equation}
\mathbb{P}\bigl( X_{n+1} = f(h) \mid X_1^n = h \bigr) = 1,
\label{eq:cond_prob_one_reverse}
\end{equation}
and for any $a \neq f(h)$,
\[
\mathbb{P}\bigl( X_{n+1} = a \mid X_1^n = h \bigr) = 0.
\]

Thus, for every positive-probability history $h$, the conditional distribution of $X_{n+1}$ given $X_1^n = h$ is degenerate, concentrating all mass on the single state $f(h)$. Consequently, by the definition of $M(h)$ in \eqref{eq:M_h}, we have
\[
M(h) = \max_{a \in \mathcal{A}} \mathbb{P}(X_{n+1}=a \mid X_1^n = h) = 1.
\]

\paragraph{Step 2: Contribution of zero-probability histories.}
For histories $h$ with $\mathbb{P}(X_1^n = h) = 0$, the value of $M(h)$ does not affect the expectation $\mathbb{E}[M_n]$, because such histories contribute zero weight in the summation defining the expectation. More formally, for any choice of $M(h) \in [0,1]$ (the natural range of a conditional probability maximum), we have
\[
\mathbb{P}(X_1^n = h) \cdot M(h) = 0 \cdot M(h) = 0.
\]
Thus these histories contribute zero to the expectation, regardless of how we define $M(h)$ for them.

\paragraph{Step 3: Computation of $\mathbb{E}[M_n]$.}
Expanding the expectation as a weighted sum over all histories (see Section~\ref{subsec:explicit_expansion}), we obtain
\begin{align}
\mathbb{E}[M_n]
&= \sum_{h \in \mathcal{A}^n} \mathbb{P}(X_1^n = h) \cdot M(h) \notag \\
&= \sum_{\substack{h \in \mathcal{A}^n \\ \mathbb{P}(X_1^n = h) > 0}} \mathbb{P}(X_1^n = h) \cdot M(h)
\;+\; \sum_{\substack{h \in \mathcal{A}^n \\ \mathbb{P}(X_1^n = h) = 0}} \mathbb{P}(X_1^n = h) \cdot M(h) \notag \\
&= \sum_{\substack{h \in \mathcal{A}^n \\ \mathbb{P}(X_1^n = h) > 0}} \mathbb{P}(X_1^n = h) \cdot 1 \;+\; 0 \notag \\
&= \sum_{\substack{h \in \mathcal{A}^n \\ \mathbb{P}(X_1^n = h) > 0}} \mathbb{P}(X_1^n = h). \label{eq:sum_over_positive}
\end{align}
The second line splits the sum into positive-probability and zero-probability histories. 
For positive-probability histories, Step~1 established that $M(h)=1$ (the conditional distribution is degenerate, concentrated on $f(h)$). 
For zero-probability histories, the term is zero as argued in Step~2.

Now observe that the collection $\{\{X_1^n = h\} : h \in \mathcal{A}^n\}$ forms a partition of $\Omega$ (up to a null set). Therefore,
\[
\sum_{\substack{h \in \mathcal{A}^n \\ \mathbb{P}(X_1^n = h) > 0}} \mathbb{P}(X_1^n = h)
= \sum_{h \in \mathcal{A}^n} \mathbb{P}(X_1^n = h) = 1,
\]
where the last equality follows from the fact that the total probability of all histories sums to $1$ (the events $\{X_1^n = h\}$ for different $h$ are mutually disjoint and their union is $\Omega$, modulo a null set). Substituting this into \eqref{eq:sum_over_positive} yields
\begin{equation}
\mathbb{E}[M_n] = 1.
\label{eq:E_Mn_eq_1_reverse}
\end{equation}

\paragraph{Step 4: Conclusion.}
Inserting $\mathbb{E}[M_n] = 1$ from \eqref{eq:E_Mn_eq_1_reverse} into the definition $\Delta_n^{\text{mem}} = \mathbb{E}[M_n] - U_n$ (see \eqref{eq:MPE_def}) gives
\[
\Delta_n^{\text{mem}} = 1 - U_n.
\]
Thus the upper bound is attained, completing the proof of the reverse implication.
\end{proof}

The equivalence established in Theorem~\ref{thm:sharpness} shows that the Memory Prediction Excess reaches its theoretical maximum precisely when the past completely determines the future — in other words, when the process is deterministic given the history. In such cases, conditioning on $X_1^n$ eliminates all uncertainty about $X_{n+1}$, so the conditional optimal prediction accuracy becomes $1$ almost surely, and the excess over the static accuracy $U_n$ is maximal.

\subsection{Degenerate Cases of the MPE}
\label{subsec:degenerate}

Having established the non-negativity and the upper bound with its sharpness condition in the preceding subsections, we now examine two scenarios in which the MPE attains its lower bound, namely zero. These cases demonstrate that $\Delta_n^{\text{mem}} = 0$ can arise either because the future carries no additional information beyond the marginal distribution (independence) or because the static prediction is already perfect (static degeneracy). The independence case reinforces the interpretation of the MPE as a measure of predictive information carried by the past, while the static degeneracy case shows that when the baseline is already perfect, no improvement is possible. Together with the non-negativity result and the sharpness condition, these degenerate cases contribute to a fuller picture of the extremal behavior of the MPE.

\begin{corollary}[Independence implies $\boldsymbol{\Delta_n^{\text{mem}} = 0}$]\label{cor:independence}
If $X_{n+1}$ is independent of the history $X_1^n$, then
\begin{equation}
\Delta_n^{\text{mem}} = 0.
\label{eq:independence_result}
\end{equation}
\end{corollary}

\begin{proof}
The assumption of independence means that for every $a \in \mathcal{A}$ and every $h \in \mathcal{A}^n$,
\[
\mathbb{P}(X_{n+1}=a,\; X_1^n = h) = \mathbb{P}(X_{n+1}=a) \; \mathbb{P}(X_1^n = h).
\]

We first consider histories with positive probability. Fix any $h \in \mathcal{A}^n$ such that $\mathbb{P}(X_1^n = h) > 0$. For each $a \in \mathcal{A}$, the elementary conditional probability is well defined and satisfies
\begin{align}
\mathbb{P}(X_{n+1}=a \mid X_1^n = h)
&= \frac{\mathbb{P}(X_{n+1}=a,\; X_1^n = h)}{\mathbb{P}(X_1^n = h)} \notag \\
&= \frac{\mathbb{P}(X_{n+1}=a) \; \mathbb{P}(X_1^n = h)}{\mathbb{P}(X_1^n = h)} \notag \\
&= \mathbb{P}(X_{n+1}=a).
\label{eq:cond_indep}
\end{align}
Thus, for every such $h$, the conditional distribution coincides with the marginal distribution. Consequently, recalling the definition of $M(h)$ from \eqref{eq:M_h} and the definition of $U_n$ from \eqref{eq:U_n_def}, we obtain
\[
M(h) = \max_{a \in \mathcal{A}} \mathbb{P}(X_{n+1}=a \mid X_1^n = h) = \max_{a \in \mathcal{A}} \mathbb{P}(X_{n+1}=a) = U_n.
\]

Now consider histories with $\mathbb{P}(X_1^n = h) = 0$. These histories occur on a null set; their contribution to any expectation is zero regardless of the value of $M(h)$, since the factor $\mathbb{P}(X_1^n = h)$ multiplies a bounded quantity.

We now compute $\mathbb{E}[M_n]$ using the expansion from \eqref{eq:E_Mn_expansion}:
\begin{align}
\mathbb{E}[M_n]
&= \sum_{h \in \mathcal{A}^n} \mathbb{P}(X_1^n = h) \cdot M(h) \notag \\
&= \sum_{\substack{h \in \mathcal{A}^n \\ \mathbb{P}(X_1^n = h) > 0}} \mathbb{P}(X_1^n = h) \cdot U_n
\;+\; \sum_{\substack{h \in \mathcal{A}^n \\ \mathbb{P}(X_1^n = h) = 0}} \mathbb{P}(X_1^n = h) \cdot M(h) \notag \\
&= U_n \sum_{\substack{h \in \mathcal{A}^n \\ \mathbb{P}(X_1^n = h) > 0}} \mathbb{P}(X_1^n = h) \;+\; 0.
\label{eq:E_Mn_indep_step}
\end{align}
Because the events $\{X_1^n = h\}$ for $h \in \mathcal{A}^n$ partition $\Omega$ (up to a null set), the sum of $\mathbb{P}(X_1^n = h)$ over all $h$ equals $1$, and the omission of zero-probability terms does not change the total. Hence,
\[
\sum_{\substack{h \in \mathcal{A}^n \\ \mathbb{P}(X_1^n = h) > 0}} \mathbb{P}(X_1^n = h) = 1.
\]
Substituting this into \eqref{eq:E_Mn_indep_step} yields $\mathbb{E}[M_n] = U_n$.

Finally, inserting $\mathbb{E}[M_n] = U_n$ into the definition of the Memory Prediction Excess (see \eqref{eq:MPE_def}) gives
\[
\Delta_n^{\text{mem}} = \mathbb{E}[M_n] - U_n = U_n - U_n = 0,
\]
which is exactly \eqref{eq:independence_result}.
\end{proof}

\begin{corollary}[Static degeneracy ($\boldsymbol{U_n = 1}$) implies $\boldsymbol{\Delta_n^{\text{mem}} = 0}$]\label{cor:static_degenerate}
If $U_n = 1$, then
\begin{equation}
\Delta_n^{\text{mem}} = 0.
\label{eq:static_degenerate_result}
\end{equation}
\end{corollary}

\begin{proof}
The condition $U_n = 1$ means that $\max_{a \in \mathcal{A}} \mathbb{P}(X_{n+1}=a) = 1$. Since the marginal probabilities are nonnegative and satisfy $\sum_{a \in \mathcal{A}} \mathbb{P}(X_{n+1}=a) = 1$, the maximum can equal $1$ only if there exists a unique state $a_0 \in \mathcal{A}$ such that
\[
\mathbb{P}(X_{n+1} = a_0) = 1,
\]
and for all $a \neq a_0$, $\mathbb{P}(X_{n+1} = a) = 0$. (If two distinct states both had probability $1$, their sum would already exceed $1$, contradicting the normalization.) In other words, $X_{n+1} = a_0$ holds almost surely.

We now examine the conditional probabilities given a positive-probability history. Fix any $h \in \mathcal{A}^n$ with $\mathbb{P}(X_1^n = h) > 0$. Because $X_{n+1} = a_0$ almost surely, the complement $\{X_{n+1} \neq a_0\}$ is a null set. Hence,
\[
\mathbb{P}(X_{n+1} \neq a_0,\; X_1^n = h) \le \mathbb{P}(X_{n+1} \neq a_0) = 0,
\]
so this quantity is zero. Consequently,
\begin{align}
\mathbb{P}(X_{n+1} = a_0,\; X_1^n = h)
&= \mathbb{P}(X_1^n = h) - \mathbb{P}(X_{n+1} \neq a_0,\; X_1^n = h) \notag \\
&= \mathbb{P}(X_1^n = h) - 0 = \mathbb{P}(X_1^n = h).
\label{eq:joint_degenerate}
\end{align}
For any $a \neq a_0$, we have $\{X_{n+1}=a\} \subseteq \{X_{n+1} \neq a_0\}$, and therefore
\[
\mathbb{P}(X_{n+1} = a,\; X_1^n = h) \le \mathbb{P}(X_{n+1} \neq a_0) = 0,
\]
which forces $\mathbb{P}(X_{n+1} = a,\; X_1^n = h) = 0$.

Using the elementary definition of conditional probability (which applies because $\mathbb{P}(X_1^n = h) > 0$), we obtain
\begin{align}
\mathbb{P}(X_{n+1} = a_0 \mid X_1^n = h)
&= \frac{\mathbb{P}(X_{n+1} = a_0,\; X_1^n = h)}{\mathbb{P}(X_1^n = h)} = \frac{\mathbb{P}(X_1^n = h)}{\mathbb{P}(X_1^n = h)} = 1, \notag \\
\mathbb{P}(X_{n+1} = a \mid X_1^n = h)
&= \frac{\mathbb{P}(X_{n+1} = a,\; X_1^n = h)}{\mathbb{P}(X_1^n = h)} = 0, \qquad a \neq a_0.
\end{align}
Thus, for every positive-probability history $h$, the conditional distribution is degenerate, concentrating all mass on the single state $a_0$. Consequently, by the definition of $M(h)$ in \eqref{eq:M_h},
\[
M(h) = \max_{a \in \mathcal{A}} \mathbb{P}(X_{n+1}=a \mid X_1^n = h) = 1.
\]

For histories with $\mathbb{P}(X_1^n = h) = 0$, their contribution to $\mathbb{E}[M_n]$ is zero, regardless of the value of $M(h)$, because the factor $\mathbb{P}(X_1^n = h)$ is zero.

We now compute $\mathbb{E}[M_n]$ using the expansion from \eqref{eq:E_Mn_expansion}:
\begin{align}
\mathbb{E}[M_n]
&= \sum_{h \in \mathcal{A}^n} \mathbb{P}(X_1^n = h) \cdot M(h) \notag \\
&= \sum_{\substack{h \in \mathcal{A}^n \\ \mathbb{P}(X_1^n = h) > 0}} \mathbb{P}(X_1^n = h) \cdot 1
\;+\; \sum_{\substack{h \in \mathcal{A}^n \\ \mathbb{P}(X_1^n = h) = 0}} \mathbb{P}(X_1^n = h) \cdot M(h) \notag \\
&= \sum_{\substack{h \in \mathcal{A}^n \\ \mathbb{P}(X_1^n = h) > 0}} \mathbb{P}(X_1^n = h).
\label{eq:E_Mn_degenerate_step}
\end{align}
As argued in the proof of Corollary~\ref{cor:independence}, the events $\{X_1^n = h\}$ partition $\Omega$ (up to a null set). Therefore, the sum of $\mathbb{P}(X_1^n = h)$ over all $h$ with positive probability equals $1$. Hence $\mathbb{E}[M_n] = 1$.

Finally, since $U_n = 1$ by hypothesis, the definition of the Memory Prediction Excess (see \eqref{eq:MPE_def}) yields
\[
\Delta_n^{\text{mem}} = \mathbb{E}[M_n] - U_n = 1 - 1 = 0,
\]
which is exactly \eqref{eq:static_degenerate_result}.
\end{proof}

\begin{remark}[Summary of extremal behaviors of $\Delta_n^{\text{mem}}$]\label{rem:extremal_behaviors}
The results established above — non-negativity (Theorem~\ref{thm:nonneg}), the upper bound (Proposition~\ref{prop:upper_bound}), its sharpness condition (Theorem~\ref{thm:sharpness}), and the two degenerate corollaries (Corollaries~\ref{cor:independence} and~\ref{cor:static_degenerate}) — together reveal that $\Delta_n^{\text{mem}}$ ranges from $0$ to $1 - U_n$, with the following extremal cases:

\begin{itemize}
    \item \textbf{Independence (no predictive gain from memory).} If $X_{n+1}$ is independent of the history $X_1^n$, then $\Delta_n^{\text{mem}} = 0$. In this scenario, the past carries no additional information about the future beyond the marginal distribution; consequently, conditioning on the history does not improve the optimal prediction accuracy.

    \item \textbf{Static degeneracy (static prediction already perfect).} If $U_n = 1$ (i.e., the static optimal predictor already achieves perfect accuracy), then again $\Delta_n^{\text{mem}} = 0$. Here the unconditional prediction is already perfect; therefore $\Delta_n^{\text{mem}} = 0$ necessarily, regardless of how informative the past might be.

    \item \textbf{Deterministic memory (past determines the future).} If $X_{n+1} = f(X_1^n)$ almost surely for some function $f: \mathcal{A}^n \to \mathcal{A}$, then $\Delta_n^{\text{mem}} = 1 - U_n$, attaining the upper bound. In this extreme case, conditioning on the history determines $X_{n+1}$ with probability one, so the conditional optimal prediction accuracy equals $1$ almost surely, and the excess over the static accuracy is maximal.
\end{itemize}

These three extremal scenarios correspond respectively to the situations where $\Delta_n^{\text{mem}} = 0$ (no benefit from memory), where $\Delta_n^{\text{mem}} = 0$ due to the static prediction being already perfect, and where $\Delta_n^{\text{mem}} = 1 - U_n$ (the past fully determines the future). Between these extremes, positive but non‑maximal values of $\Delta_n^{\text{mem}}$ indicate that historical information provides a partial improvement in prediction accuracy. Thus $\Delta_n^{\text{mem}}$ serves as a self‑contained quantitative measure of the predictive value of memory, ranging from $0$ to $1 - U_n$.
\end{remark}

\section{NORMALIZED MEMORY PREDICTION EXCESS AND ITS INFORMATION-THEORETIC CONTEXT}
\label{sec:normalized}

In the preceding sections we have established the Memory Prediction Excess $\Delta_n^{\text{mem}}$ (Definition~\ref{def:MPE}) as a quantitative measure of the average gain in optimal prediction accuracy attributable to the historical information $X_1^n$. Theorem~\ref{thm:nonneg} guarantees its non-negativity, while Proposition~\ref{prop:upper_bound} provides the upper bound $1-U_n$. Consequently we have
\begin{equation}
0 \le \Delta_n^{\text{mem}} \le 1 - U_n.
\label{eq:Delta_range}
\end{equation}

Inequality \eqref{eq:Delta_range} shows that the range and magnitude of \(\Delta_n^{\text{mem}}\) are influenced by the static optimal accuracy \(U_n\). When \(U_n\) is close to \(1\), the upper bound \(1-U_n\) becomes tight, leaving little room for \(\Delta_n^{\text{mem}}\) to attain large absolute values, regardless of how informative the history may be relative to the maximal possible gain. When \(U_n\) is farther from \(1\), this bound is less restrictive, and \(\Delta_n^{\text{mem}}\) may reflect a wider range of possible improvements. Thus, to separate the relative contribution of the historical information from the scaling effect of \(U_n\), and to obtain a quantity that is comparable across different contexts, we further introduce a normalized version of the Memory Prediction Excess (see Definition~\ref{def:eff} below).

\subsection{Definition of Normalized Memory Prediction Excess}
\label{subsec:eff_definition}

The results of Section~\ref{sec:properties} show that the MPE is bounded between $0$ and $1 - U_n$, with the upper bound depending on the static accuracy $U_n$, which varies with the time index and the process under consideration. The MPE itself captures the absolute gain, and a normalized version offers a complementary perspective by measuring the gain relative to the maximal possible improvement. This normalized quantity, denoted by \(\Delta_n^{\text{eff}}\), rescales \(\Delta_n^{\text{mem}}\) by \(1 - U_n\) (when \(U_n < 1\)), yielding a dimensionless measure that takes values in the unit interval. Its extremal values will be characterized in the following subsections, and its relation to classical information-theoretic quantities will be discussed in Subsection~\ref{subsec:info_relation}.

\begin{definition}[Normalized Memory Prediction Excess]\label{def:eff}
For a fixed time index $n \ge 1$, the \textbf{Normalized Memory Prediction Excess}, denoted by $\Delta_n^{\text{eff}}$, is defined as follows:
\begin{equation}
\Delta_n^{\text{eff}} := 
\begin{cases}
\dfrac{\Delta_n^{\text{mem}}}{1 - U_n}, & \text{if } U_n < 1,\\[8pt]
0, & \text{if } U_n = 1.
\end{cases}
\label{eq:eff_def}
\end{equation}
Equivalently, for $U_n < 1$,
\[
\Delta_n^{\text{eff}} = \frac{\mathbb{E}[M_n] - U_n}{1 - U_n}.
\]
The case $U_n = 1$ is treated separately because the denominator $1-U_n$ vanishes; Corollary~\ref{cor:static_degenerate} shows that then $\Delta_n^{\text{mem}} = 0$, so the assignment $\Delta_n^{\text{eff}} = 0$ is natural and consistent with the interpretation that when the static prediction is already perfect, no historical information can improve it. Remark~\ref{rem:eff_continuity} below provides further justification for this convention.
\end{definition}

Having defined the Normalized Memory Prediction Excess \(\Delta_n^{\text{eff}}\), we now turn to its basic properties. We first identify the scenarios in which this normalized quantity attains its extremal values, namely zero and one, and then show that \(\Delta_n^{\text{eff}}\) lies in the unit interval. We also discuss its interpretation as a measure of predictive efficiency and its relation to the original Memory Prediction Excess \(\Delta_n^{\text{mem}}\).

\begin{proposition}[Sufficient conditions for $\Delta_n^{\text{eff}} = 0$]\label{prop:eff_zero_criteria}
The normalized Memory Prediction Excess vanishes, i.e., $\Delta_n^{\text{eff}} = 0$, if either of the following conditions holds:
\begin{enumerate}
    \item $U_n = 1$ (static degeneracy);
    \item $X_{n+1}$ is independent of $X_1^n$ (denoted $X_{n+1} \perp X_1^n$).
\end{enumerate}
These two scenarios are precisely the natural degenerate cases in which the historical information does not improve upon the static prediction.
\end{proposition}

\begin{proof}
We verify the two sufficient conditions separately. For each condition, we first recall the corresponding result already established for $\Delta_n^{\text{mem}}$, and then deduce the consequence for $\Delta_n^{\text{eff}}$.

\paragraph{Condition (i): $U_n = 1$ (static degeneracy).} 
Corollary~\ref{cor:static_degenerate} shows that when $U_n = 1$, we have $\Delta_n^{\text{mem}} = 0$. Moreover, by Definition~\ref{def:eff}, the case $U_n = 1$ is treated separately in the piecewise definition, yielding directly $\Delta_n^{\text{eff}} = 0$ (the second branch of \eqref{eq:eff_def}). Hence the conclusion holds in this case.

\paragraph{Condition (ii): $X_{n+1}$ is independent of $X_1^n$.} 
We proceed in two steps.

\textit{Step 1. Recalling the consequence for $\Delta_n^{\text{mem}}$.} 
Corollary~\ref{cor:independence} establishes that under the independence assumption, $\Delta_n^{\text{mem}} = 0$. For the reader's convenience, we recall the key idea of that proof: for every positive-probability history $h \in \mathcal{A}^n$, independence gives
\[
\mathbb{P}(X_{n+1}=a \mid X_1^n = h) = \mathbb{P}(X_{n+1}=a) \qquad \forall a \in \mathcal{A},
\]
which implies $M(h) = U_n$ and consequently $\mathbb{E}[M_n] = U_n$. Histories with zero probability contribute nothing to the expectation. Hence $\Delta_n^{\text{mem}} = \mathbb{E}[M_n] - U_n = 0$.

\textit{Step 2. From $\Delta_n^{\text{mem}} = 0$ to $\Delta_n^{\text{eff}} = 0$.} 
Now we compute $\Delta_n^{\text{eff}}$ using Definition~\ref{def:eff}:
\begin{itemize}
    \item If $U_n = 1$, we are already covered by Condition~(i) above, and $\Delta_n^{\text{eff}} = 0$ by direct definition.
    \item If $U_n < 1$, then $1 - U_n > 0$. Substituting $\Delta_n^{\text{mem}} = 0$ into the first branch of \eqref{eq:eff_def} yields
    \[
    \Delta_n^{\text{eff}} = \frac{\Delta_n^{\text{mem}}}{1 - U_n} = \frac{0}{1 - U_n} = 0.
    \]
\end{itemize}

In either subcase, we conclude $\Delta_n^{\text{eff}} = 0$.

Thus both conditions (i) and (ii) are sufficient for $\Delta_n^{\text{eff}} = 0$, completing the proof.
\end{proof}

\begin{proposition}[Characterization of $\Delta_n^{\text{eff}} = 1$]\label{prop:eff_one_criteria}
Assume $U_n < 1$. Then the following statements are equivalent:
\begin{enumerate}
    \item $\Delta_n^{\text{eff}} = 1$;
    \item $\Delta_n^{\text{mem}} = 1 - U_n$;
    \item There exists a function $f: \mathcal{A}^n \to \mathcal{A}$ such that $X_{n+1} = f(X_1^n)$ almost surely.
\end{enumerate}
\end{proposition}

\begin{proof}
The assumption $U_n < 1$ guarantees $1 - U_n > 0$, so division by $1 - U_n$ is legitimate and the denominator does not vanish. We establish the two equivalences separately.

\paragraph{Equivalence between (i) and (ii).} 
From Definition~\ref{def:eff}, for $U_n < 1$ we have
\[
\Delta_n^{\text{eff}} = \frac{\Delta_n^{\text{mem}}}{1 - U_n}.
\]
Since $1 - U_n > 0$, multiplying both sides by this positive quantity preserves the equivalence:
\[
\Delta_n^{\text{eff}} = 1 \;\Longleftrightarrow\; \Delta_n^{\text{mem}} = 1 - U_n.
\]

\paragraph{Equivalence between (ii) and (iii).} 
Theorem~\ref{thm:sharpness} (see Section~\ref{sec:properties}) establishes precisely that $\Delta_n^{\text{mem}} = 1 - U_n$ holds if and only if there exists a deterministic function $f: \mathcal{A}^n \to \mathcal{A}$ such that $X_{n+1} = f(X_1^n)$ almost surely. This result was proved in full detail in Section~\ref{subsec:upper_bound}, where the equivalence was shown by analyzing the condition $\mathbb{E}[M_n] = 1$ and its consequences for the conditional probability distributions.

Combining the two equivalences yields the chain
\[
\Delta_n^{\text{eff}} = 1 \;\Longleftrightarrow\; \Delta_n^{\text{mem}} = 1 - U_n \;\Longleftrightarrow\; X_{n+1} = f(X_1^n) \text{ a.s.},
\]
which completes the proof.
\end{proof}

\begin{remark}
Proposition~\ref{prop:eff_one_criteria} shows that $\Delta_n^{\text{eff}} = 1$ occurs precisely when $X_{n+1}$ is almost surely a function of $X_1^n$. In this case the conditional prediction accuracy satisfies $M_n = 1$ almost surely, while the static accuracy $U_n$ satisfies $0 \le U_n < 1$. Hence the historical information realizes the maximal possible proportion of the achievable gain within its permissible range.
\end{remark}

From Definition~\ref{def:eff} together with the bounds established in Proposition~\ref{prop:upper_bound} and the characterizations given in Propositions~\ref{prop:eff_zero_criteria} and~\ref{prop:eff_one_criteria}, we obtain the following range for the normalized quantity:
\begin{equation}
0 \le \Delta_n^{\text{eff}} \le 1.
\label{eq:eff_range}
\end{equation}

Thus $\Delta_n^{\text{eff}}$ is a dimensionless number confined to the unit interval. The extremal values are related to the propositions above as follows:
\begin{itemize}
    \item $\Delta_n^{\text{eff}} = 0$ occurs, in particular, in the two scenarios identified in Proposition~\ref{prop:eff_zero_criteria}: independence of $X_{n+1}$ from $X_1^n$ (Corollary~\ref{cor:independence}) or static degeneracy $U_n = 1$ (Corollary~\ref{cor:static_degenerate}).
    \item $\Delta_n^{\text{eff}} = 1$ occurs exactly under the deterministic condition of Proposition~\ref{prop:eff_one_criteria}, namely $X_{n+1} = f(X_1^n)$ almost surely for some function $f: \mathcal{A}^n \to \mathcal{A}$ (see Theorem~\ref{thm:sharpness}).
\end{itemize}

Intermediate values $0 < \Delta_n^{\text{eff}} < 1$ reflect situations where the historical information yields a partial improvement, and $\Delta_n^{\text{eff}}$ measures the proportion of the theoretically maximal gain that is realized. Because $\Delta_n^{\text{eff}}$ eliminates the influence of $U_n$, it enables comparisons of the predictive efficiency of memory across different time indices $n$ or across different stochastic processes. For this reason we may also refer to $\Delta_n^{\text{eff}}$ as the \textbf{predictive efficiency} of the memory $X_1^n$ with respect to $X_{n+1}$. (The superscript ``eff'' in $\Delta_n^{\text{eff}}$ stands for ``efficiency'', consistent with this interpretation.)

\begin{remark}[Relation to the original MPE]\label{rem:eff_relation}
From Definition~\ref{def:eff} we have the simple relation
\[
\Delta_n^{\text{mem}} = \Delta_n^{\text{eff}} \cdot (1 - U_n),
\]
which holds for all $n$ (with the convention that $0 \cdot 0 = 0$ when $U_n = 1$). Thus $\Delta_n^{\text{eff}}$ represents the proportion of the maximal possible excess $1-U_n$ that is actually achieved.
\end{remark}

\begin{remark}[On the convention for $U_n = 1$]\label{rem:eff_continuity}
When $U_n = 1$, the denominator $1-U_n$ vanishes. Corollary~\ref{cor:static_degenerate} gives $\Delta_n^{\text{mem}} = 0$, so the expression $\Delta_n^{\text{mem}}/(1-U_n)$ is of the ambiguous form $0/0$ and is not defined by the ordinary rules of arithmetic. In this boundary case we adopt the convention $\Delta_n^{\text{eff}} = 0$. This choice is justified as follows.

\paragraph{Conceptual justification.}
If $U_n = 1$, the static optimal predictor achieves perfect accuracy with probability one. No historical information can improve upon a prediction that is already perfect; therefore the normalized excess, which measures the proportion of the maximal possible gain that is realized, ought to be zero.

\paragraph{Consistency justification.}
For $U_n < 1$, we have the relation $\Delta_n^{\text{mem}} = \Delta_n^{\text{eff}} \cdot (1-U_n)$. When $U_n = 1$, the right-hand side becomes $\Delta_n^{\text{eff}} \cdot 0$, which equals $0$ independently of the value assigned to $\Delta_n^{\text{eff}}$. Hence this relation alone does not constrain the choice of $\Delta_n^{\text{eff}}$.

The quantity $\Delta_n^{\text{eff}}$ is intended to measure the ratio of the excess $\Delta_n^{\text{mem}}$ to the maximal possible excess $1-U_n$. When $U_n = 1$, the maximal possible excess $1-U_n$ is zero, and Corollary~\ref{cor:static_degenerate} gives $\Delta_n^{\text{mem}} = 0$. In this situation, a natural value for this ratio is $0$. We therefore adopt the convention $\Delta_n^{\text{eff}} = 0$.
\end{remark}

\subsection{Relation to Information-Theoretic Quantities}
\label{subsec:info_relation}

The MPE (Definition~\ref{def:MPE}) and its normalized version (Definition~\ref{def:eff}) are defined in terms of conditional probabilities and prediction accuracies. This formulation offers a perspective that sits alongside information-theoretic measures such as entropy, mutual information, and transfer entropy. The latter focus on information content, dependence, or information flow. The MPE centers on the improvement in the probability of a correct prediction. These perspectives are complementary: each captures a distinct facet of the relationship between past and future. In the following we briefly discuss several well-established information-theoretic concepts and their connections to the quantities introduced in this work.

\paragraph{Shannon entropy and mutual information.}
The foundational work of Shannon \cite{shannon1948mathematic} introduced entropy $H(X) = -\sum_x \mathbb{P}(X=x)\log\mathbb{P}(X=x)$ as a measure of uncertainty and mutual information
\[
I(X_{n+1}; X_1^n) = H(X_{n+1}) - H(X_{n+1} \mid X_1^n)
\]
as a measure of the reduction in uncertainty about $X_{n+1}$ provided by the history $X_1^n$. Mutual information is symmetric and non-negative, and it vanishes if and only if $X_{n+1}$ and $X_1^n$ are independent. Mutual information quantifies, in an average information-theoretic sense, how much the past tells us about the future. The probability of correctly predicting the actual outcome $X_{n+1}$ under an optimal prediction rule — which is the focus of our quantity $\Delta_n^{\text{mem}}$ — addresses a different aspect from that quantified by mutual information. The two perspectives are complementary: mutual information captures the average reduction in logarithmic uncertainty, while $\Delta_n^{\text{mem}}$ captures the average improvement in the optimal probability of a correct guess.

\paragraph{Transfer entropy.}
The work \cite{schreiber2000measuring} introduced transfer entropy as an information-theoretic measure that quantifies the statistical coherence between systems evolving in time. As noted in that work, the standard time-delayed mutual information does not distinguish information that is actually exchanged from shared information due to common history or input signals. By appropriate conditioning of transition probabilities, transfer entropy excludes these influences. The resulting measure is able to distinguish effectively driving and responding elements and to detect asymmetry in the interaction of subsystems.

Our Memory Prediction Excess takes a different perspective: it conditions on the entire history \(X_1^n\) and measures the improvement in optimal prediction accuracy. The two measures address different questions: transfer entropy focuses on asymmetric information flow and the direction of influence between subsystems, and \(\Delta_n^{\text{mem}}\) asks by how much the probability of the best possible prediction increases when we have access to the history \(X_1^n\).

\paragraph{Measures of dependence.}
The work \cite{renyi1959measures} formulated a set of postulates for a measure of dependence between two random variables and studied the maximal correlation as a quantity satisfying these postulates. This framework focuses on the strength of statistical association between two variables. Our \(\Delta_n^{\text{mem}}\) is inherently directional, as it concerns the predictive relationship from the past to the future, and is rooted in prediction accuracy rather than correlation or mutual dependence. In this sense, \(\Delta_n^{\text{mem}}\) offers a complementary perspective to these dependence measures: it quantifies the benefit of knowing the past for the purpose of forecasting in terms of optimal prediction accuracy.

\paragraph{Information bottleneck and representation learning.}
The work \cite{shwartz2017opening} followed up on the idea of analyzing deep neural networks through the information plane and demonstrated the effectiveness of this visualization. Their main findings include that most training epochs are spent on compressing the input into efficient representations, and that converged layers lie on or near the information bottleneck theoretical bound. The compression phase begins once the training error becomes small, after which stochastic gradient descent transitions into a stochastic relaxation constrained by the training error. The authors also observe that the advantage of additional hidden layers is primarily computational, as they reduce the relaxation time. Our normalized Memory Prediction Excess \(\Delta_n^{\text{eff}}\) captures a notion of efficiency — namely, the proportion of maximal possible predictive gain that historical information provides.

\paragraph{Conditional and marginal mutual information in time series.}
The work \cite{chavez2016conditional} studied the marginal mutual information \(I(X_{n+1}; X_1^n)\) and the conditional mutual information, which quantifies newly available information for prediction. In particular, predictability in Gaussian processes was characterized in terms of sensitivity to random shocks, and series expansions for the mutual information were derived in terms of autocorrelation or moving average parameters. Our setting concerns general discrete-time finite-state processes and does not assume Gaussianity, stationarity, or parametric forms. Our focus is on optimal prediction accuracy, specifically the probability of a correct guess, while mutual information quantifies dependence in logarithmic units. The two approaches are mutually compatible: mutual information provides an average logarithmic measure of dependence, and \(\Delta_n^{\text{mem}}\) provides a direct probabilistic measure of the best achievable prediction performance.

The brief survey above recalls several well-known measures for quantifying dependence and prediction in stochastic processes. Entropy, mutual information, transfer entropy, measures of dependence, the information bottleneck, and conditional mutual information all capture certain aspects of the relationship between past and future, operating at the level of information content, correlation, or logarithmic uncertainty. Our quantities \(\Delta_n^{\text{mem}}\) and \(\Delta_n^{\text{eff}}\) operate at the level of prediction accuracy, specifically the probability of correctly guessing the next outcome when one selects the most likely state given the available information. We view these tools as complementary, each addressing a different facet of the predictive relationship. A comprehensive understanding of prediction in stochastic processes benefits from the availability of multiple conceptually distinct rigorous quantitative tools.

\subsection{A Lower Bound for the Memory Prediction Excess}
\label{subsec:MPE_lower_bound}

Having obtained an upper bound for the MPE in Subsection~\ref{subsec:upper_bound}, we now consider the corresponding question of a lower bound. This bound is obtained by comparing predictions for the same future state using histories of different lengths, and it reflects the intuitive fact that additional past information should not reduce the expected optimal prediction accuracy. Specifically, it shows that the MPE is at least as large as the excess already achievable with a shorter history. This result also serves as a bridge to the finite-history setting of Section~\ref{sec:finite_history}, where the behavior of the MPE as a function of the memory length is examined in detail.

In Proposition~\ref{prop:upper_bound} we established an upper bound for the Memory Prediction Excess:
\[
\Delta_n^{\text{mem}} \le 1 - U_n, \qquad n \ge 1.
\]
This inequality shows that $\Delta_n^{\text{mem}}$ cannot exceed the gap between perfect prediction and the static optimal accuracy. A natural complementary question is whether one can also obtain a lower bound for $\Delta_n^{\text{mem}}$. In this subsection we derive such a lower bound by comparing predictions for the same future variable using histories of different lengths. Specifically, we show that $\Delta_{n+1}^{\text{mem}}$ is at least as large as the excess (relative to the static accuracy) attainable with the shorter history $X_1^n$ when predicting $X_{n+2}$.

We first recall the definition of the Memory Prediction Excess (Definition~\ref{def:MPE}):
\[
\Delta_n^{\text{mem}} = \mathbb{E}[M_n] - U_n, \qquad n \ge 1,
\]
where $M_n = \max_{a \in \mathcal{A}} \mathbb{P}(X_{n+1} = a \mid X_1^n)$ (see \eqref{eq:M_n_def}) and $U_n = \max_{a \in \mathcal{A}} \mathbb{P}(X_{n+1} = a)$ (see \eqref{eq:U_n_def}).

\begin{proposition}[Lower bound for $\Delta_{n+1}^{\text{mem}}$]\label{prop:MPE_lower_bound}
For every $n \ge 1$,
\begin{equation}
\Delta_{n+1}^{\text{mem}} \;\ge\; \mathbb{E}\Bigl[ \max_{a \in \mathcal{A}} \mathbb{P}(X_{n+2}=a \mid X_1^n) \Bigr] \;-\; U_{n+1},
\label{eq:MPE_lower_bound}
\end{equation}
where $U_{n+1} = \max_{a \in \mathcal{A}} \mathbb{P}(X_{n+2}=a)$.
\end{proposition}

\begin{proof}
We prove the inequality by establishing a monotonicity property of conditional optimal prediction accuracies and then relating it to the definition of $\Delta_{n+1}^{\text{mem}}$. The proof proceeds in four steps.

\paragraph{Step 1: Relating conditional probabilities given different history lengths.}

Fix an arbitrary state $a \in \mathcal{A}$. Define the indicator function of the event $\{X_{n+2}=a\}$:
\[
\mathbf{1}_{\{X_{n+2}=a\}}(\omega) := \begin{cases}
1, & \text{if } X_{n+2}(\omega) = a,\\
0, & \text{if } X_{n+2}(\omega) \neq a.
\end{cases}
\]

Recall that for any $\sigma$-algebra $\mathcal{G}$, the conditional probability $\mathbb{P}(X_{n+2}=a \mid \mathcal{G})$ is defined as the conditional expectation of the indicator function:
\[
\mathbb{P}(X_{n+2}=a \mid \mathcal{G}) = \mathbb{E}\bigl[ \mathbf{1}_{\{X_{n+2}=a\}} \mid \mathcal{G} \bigr].
\]
All equalities and inequalities involving conditional expectations are understood to hold $\mathbb{P}$-almost surely (i.e., up to a null set). In what follows, we shall often omit the explicit ``almost surely'' when the context makes it clear.

In probability theory, conditioning on a random vector $Y$ is interpreted as conditioning on the $\sigma$-algebra $\sigma(Y)$ generated by $Y$; that is,
\[
\mathbb{P}(\cdot \mid Y) = \mathbb{P}(\cdot \mid \sigma(Y)), \qquad
\mathbb{E}[\cdot \mid Y] = \mathbb{E}[\cdot \mid \sigma(Y)].
\]
Consequently, we have
\begin{align}
\mathbb{P}(X_{n+2}=a \mid X_1^{n+1}) &= \mathbb{E}\bigl[ \mathbf{1}_{\{X_{n+2}=a\}} \mid X_1^{n+1} \bigr], \label{eq:cond_n+1} \\
\mathbb{P}(X_{n+2}=a \mid X_1^n) &= \mathbb{E}\bigl[ \mathbf{1}_{\{X_{n+2}=a\}} \mid X_1^n \bigr]. \label{eq:cond_n}
\end{align}

Because $X_1^n$ consists of the first $n$ components of $X_1^{n+1}$, the $\sigma$-algebra generated by $X_1^n$ is contained in that generated by $X_1^{n+1}$:
\[
\sigma(X_1^n) \subseteq \sigma(X_1^{n+1}).
\]

The tower property (also known as the law of total expectation) is a fundamental property of conditional expectations: if $\mathcal{H} \subseteq \mathcal{G}$ are $\sigma$-algebras, then for any integrable random variable $Z$,
\[
\mathbb{E}\bigl[ \mathbb{E}[Z \mid \mathcal{G}] \mid \mathcal{H} \bigr] = \mathbb{E}[Z \mid \mathcal{H}].
\]

Applying this property with $Z = \mathbf{1}_{\{X_{n+2}=a\}}$, $\mathcal{G} = \sigma(X_1^{n+1})$, and $\mathcal{H} = \sigma(X_1^n)$, we obtain
\[
\mathbb{E}\Bigl[ \mathbb{E}\bigl[ \mathbf{1}_{\{X_{n+2}=a\}} \mid \sigma(X_1^{n+1}) \bigr] \;\Big|\; \sigma(X_1^n) \Bigr] = \mathbb{E}\bigl[ \mathbf{1}_{\{X_{n+2}=a\}} \mid \sigma(X_1^n) \bigr].
\]

Substituting \eqref{eq:cond_n+1} and \eqref{eq:cond_n} into the above equality yields the fundamental relation
\begin{equation}
\mathbb{E}\Bigl[ \mathbb{P}(X_{n+2}=a \mid X_1^{n+1}) \;\Big|\; X_1^n \Bigr] = \mathbb{P}(X_{n+2}=a \mid X_1^n). \label{eq:tower_relation}
\end{equation}
This relation expresses that the conditional probability given the coarser information $X_1^n$ is precisely the conditional expectation (given the same coarser information) of the conditional probability given the finer information $X_1^{n+1}$.

\paragraph{Step 2: Introducing notation and applying monotonicity of conditional expectation.}

Let us fix a version of the conditional probabilities such that all relevant properties hold pointwise on a set $\Omega_0 \subseteq \Omega$ with $\mathbb{P}(\Omega_0) = 1$. For $\omega \in \Omega_0$, define
\[
q_a^{(n+1)}(\omega) := \mathbb{P}(X_{n+2}=a \mid X_1^{n+1})(\omega), \qquad a \in \mathcal{A}.
\]
For $\omega \notin \Omega_0$, we may define $q_a^{(n+1)}(\omega)$ arbitrarily (e.g., as $1/|\mathcal{A}|$), as such choices do not affect any expectation or almost-sure statement.

On $\Omega_0$, equation \eqref{eq:tower_relation} holds pointwise, and for each $a \in \mathcal{A}$,
\[
\mathbb{P}(X_{n+2}=a \mid X_1^n)(\omega) = \mathbb{E}\bigl[ q_a^{(n+1)} \mid X_1^n \bigr](\omega).
\]

Now consider the optimal prediction accuracy based on the shorter history $X_1^n$. From the relation above, we have
\[
\max_{a \in \mathcal{A}} \mathbb{P}(X_{n+2}=a \mid X_1^n)(\omega) = \max_{a \in \mathcal{A}} \mathbb{E}\bigl[ q_a^{(n+1)} \mid X_1^n \bigr](\omega), \qquad \omega \in \Omega_0.
\]

For any fixed $a \in \mathcal{A}$, the pointwise inequality
\[
q_a^{(n+1)}(\omega) \le \max_{b \in \mathcal{A}} q_b^{(n+1)}(\omega)
\]
holds for every $\omega \in \Omega_0$, and hence almost surely.

Conditional expectation is monotone: if $X \le Y$ almost surely, then for any $\sigma$-algebra $\mathcal{G}$,
\[
\mathbb{E}[X \mid \mathcal{G}] \le \mathbb{E}[Y \mid \mathcal{G}] \quad \text{almost surely}.
\]

Applying monotonicity with $X = q_a^{(n+1)}$, $Y = \max_{b \in \mathcal{A}} q_b^{(n+1)}$, and $\mathcal{G} = \sigma(X_1^n)$, we obtain for each $a \in \mathcal{A}$,
\[
\mathbb{E}\bigl[ q_a^{(n+1)} \mid X_1^n \bigr] \le \mathbb{E}\bigl[ \max_{b \in \mathcal{A}} q_b^{(n+1)} \mid X_1^n \bigr] \quad \text{almost surely}.
\]

Since $\mathcal{A}$ is finite, we may take the maximum over $a \in \mathcal{A}$ on the left-hand side. For a finite collection of inequalities $x_a \le C$ (where $C$ is the right-hand side, which does not depend on $a$), the maximum of the left-hand sides also satisfies $\max_a x_a \le C$. More precisely, because the inequality holds for each $a$ individually on a set of full measure (the intersection of finitely many almost-sure events is still almost sure), we have
\begin{equation}
\max_{a \in \mathcal{A}} \mathbb{E}\bigl[ q_a^{(n+1)} \mid X_1^n \bigr] \le \mathbb{E}\bigl[ \max_{b \in \mathcal{A}} q_b^{(n+1)} \mid X_1^n \bigr] \quad \text{almost surely}. \label{eq:max_inequality}
\end{equation}

\paragraph{Step 3: Recognizing the optimal prediction accuracy with the longer history.}

Recall from Definition~\ref{def:MPE} (see also equations \eqref{eq:M_h} and \eqref{eq:M_n_def}) that for a given history length $n$, the optimal prediction accuracy for $X_{n+1}$ based on $X_1^n$ is the random variable
\[
M_n = \max_{a \in \mathcal{A}} \mathbb{P}(X_{n+1}=a \mid X_1^n).
\]
Following the same pattern, for the longer history $X_1^{n+1}$, we define the optimal prediction accuracy for $X_{n+2}$ as
\[
M_{n+1} := \max_{a \in \mathcal{A}} \mathbb{P}(X_{n+2}=a \mid X_1^{n+1}).
\]
This is a natural extension of the notation: $M_{n+1}$ is the random variable representing the best possible accuracy when predicting $X_{n+2}$ using the history $X_1^{n+1}$.

By construction, $\max_{b \in \mathcal{A}} q_b^{(n+1)} = M_{n+1}$ holds pointwise, since $q_b^{(n+1)} = \mathbb{P}(X_{n+2}=b \mid X_1^{n+1})$ by definition.

Now recall the inequality established at the end of Step~2: $\max_{a \in \mathcal{A}} \mathbb{E}[ q_a^{(n+1)} \mid X_1^n ] \le \mathbb{E}[ \max_{b \in \mathcal{A}} q_b^{(n+1)} \mid X_1^n ]$ almost surely. Replacing $\max_{b \in \mathcal{A}} q_b^{(n+1)}$ with $M_{n+1}$ gives $\max_{a \in \mathcal{A}} \mathbb{E}[ q_a^{(n+1)} \mid X_1^n ] \le \mathbb{E}[M_{n+1} \mid X_1^n]$ almost surely. From Step~2 we also have the pointwise equality on $\Omega_0$: $\max_{a \in \mathcal{A}} \mathbb{P}(X_{n+2}=a \mid X_1^n)(\omega) = \max_{a \in \mathcal{A}} \mathbb{E}[ q_a^{(n+1)} \mid X_1^n ](\omega)$ for every $\omega \in \Omega_0$.

Combining this equality with the inequality above (and noting that $\mathbb{P}(\Omega_0)=1$), we obtain
\begin{equation}
\max_{a \in \mathcal{A}} \mathbb{P}(X_{n+2}=a \mid X_1^n) \le \mathbb{E}[M_{n+1} \mid X_1^n] \quad \text{almost surely}.
\label{eq:pointwise_bound}
\end{equation}

This inequality has a clear intuitive meaning: for almost every realization of the shorter history $X_1^n$, the optimal prediction accuracy achievable with that shorter history does not exceed the conditional expectation (given that same history) of the optimal prediction accuracy achievable with the longer history.

\paragraph{Step 4: Taking expectations and connecting to the Memory Prediction Excess.}

Take expectations on both sides of \eqref{eq:pointwise_bound}. By the law of total expectation (also known as the tower property), for any random variable $Y$ and any $\sigma$-algebra $\mathcal{G}$, we have
\[
\mathbb{E}\bigl[ \mathbb{E}[Y \mid \mathcal{G}] \bigr] = \mathbb{E}[Y].
\]
(This is the special case of the tower property where the outer conditioning is taken with respect to the trivial $\sigma$-algebra $\{\emptyset, \Omega\}$, which reduces the conditional expectation to an ordinary expectation.)

Applying this identity with $Y = M_{n+1}$ and $\mathcal{G} = \sigma(X_1^n)$ gives
\[
\mathbb{E}\bigl[ \mathbb{E}[M_{n+1} \mid X_1^n] \bigr] = \mathbb{E}[M_{n+1}].
\]

Thus, taking expectations in \eqref{eq:pointwise_bound} yields
\begin{equation}
\mathbb{E}\Bigl[ \max_{a \in \mathcal{A}} \mathbb{P}(X_{n+2}=a \mid X_1^n) \Bigr] \le \mathbb{E}[M_{n+1}].
\label{eq:expectation_bound}
\end{equation}

Now recall the definition of the Memory Prediction Excess (Definition~\ref{def:MPE}) for index $n+1$:
\[
\Delta_{n+1}^{\text{mem}} = \mathbb{E}[M_{n+1}] - U_{n+1},
\]
where $U_{n+1} = \max_{a \in \mathcal{A}} \mathbb{P}(X_{n+2}=a)$ is the static optimal prediction accuracy for $X_{n+2}$.

Since $U_{n+1}$ is a constant, subtracting it from both sides of inequality \eqref{eq:expectation_bound} preserves the inequality:
\[
\mathbb{E}\Bigl[ \max_{a \in \mathcal{A}} \mathbb{P}(X_{n+2}=a \mid X_1^n) \Bigr] - U_{n+1} \le \mathbb{E}[M_{n+1}] - U_{n+1}.
\]

The right-hand side is precisely $\Delta_{n+1}^{\text{mem}}$ by definition. Rewriting the inequality so that $\Delta_{n+1}^{\text{mem}}$ appears on the left-hand side (i.e., exchanging the two sides) yields the desired lower bound:
\[
\Delta_{n+1}^{\text{mem}} \ge \mathbb{E}\Bigl[ \max_{a \in \mathcal{A}} \mathbb{P}(X_{n+2}=a \mid X_1^n) \Bigr] - U_{n+1}.
\]

This inequality is exactly the statement of Proposition~\ref{prop:MPE_lower_bound}. Hence the proof is complete.
\end{proof}

\begin{remark}[Comparison with the upper bound and consistency with the boundary case]
\label{rem:lower_bound_interpretation}

Proposition~\ref{prop:MPE_lower_bound} provides a lower bound for $\Delta_{n+1}^{\text{mem}}$ in terms of the expected optimal prediction accuracy achievable with the shorter history $X_1^n$. Together with the upper bound $\Delta_n^{\text{mem}} \le 1 - U_n$ from Proposition~\ref{prop:upper_bound}, we now have both upper and lower estimates for the Memory Prediction Excess.

The lower bound admits a natural interpretation: it states that the gain from using $n+1$ steps of history cannot be smaller than the excess (relative to the static accuracy) already achievable with only $n$ steps of history when predicting the same future state. This property is consistent with the intuition that more information cannot hurt prediction on average.

\paragraph{Consistency with the boundary case $n = 0$.}

For completeness, we examine the boundary case $n = 0$. Substituting $n = 0$ into \eqref{eq:MPE_lower_bound} would yield
\[
\Delta_1^{\text{mem}} \ge \mathbb{E}\Bigl[ \max_{a \in \mathcal{A}} \mathbb{P}(X_{2}=a \mid X_1^0) \Bigr] - U_1.
\]
By the convention regarding the empty history established in Definition~\ref{def:MPE}, $X_1^0$ represents the empty history, so that conditioning on $X_1^0$ reduces to the unconditional probability:
\[
\mathbb{P}(X_2=a \mid X_1^0) = \mathbb{P}(X_2=a), \qquad a \in \mathcal{A}.
\]
Hence $\max_{a \in \mathcal{A}} \mathbb{P}(X_2=a \mid X_1^0) = \max_{a \in \mathcal{A}} \mathbb{P}(X_2=a) = U_1$, where the last equality follows from the definition $U_n := \max_{a \in \mathcal{A}} \mathbb{P}(X_{n+1}=a)$ (see Section~\ref{sec:definition}).

The right-hand side of the inequality therefore becomes $U_1 - U_1 = 0$, and the inequality reduces to $\Delta_1^{\text{mem}} \ge 0$. This is precisely the non‑negativity result already proved in Theorem~\ref{thm:nonneg}. Thus the $n = 0$ case is consistent with Proposition~\ref{prop:MPE_lower_bound}: the lower bound, when applied with the empty history, reduces to a known result and does not introduce any contradiction.

In summary, Proposition~\ref{prop:MPE_lower_bound} and the earlier upper bound from Proposition~\ref{prop:upper_bound} together provide a pair of estimates that bracket the Memory Prediction Excess from above and below. The lower bound is consistent with the intuition that additional historical information improves, or at least does not impair, the expected optimal prediction accuracy.
\end{remark}

\section{FINITE-HISTORY MEMORY PREDICTION EXCESS}
\label{sec:finite_history}

In the preceding sections, the Memory Prediction Excess $\Delta_n^{\text{mem}}$ was defined using the full history $X_1^n = (X_1, X_2, \dots, X_n)$ as the conditioning information for predicting $X_{n+1}$ (see Definition~\ref{def:MPE}). This formulation captures the maximal predictive gain achievable when all available past observations are exploited. A natural extension of this framework is to investigate the extent to which this gain depends on the length of the retained history. This question is motivated by the observation that, in certain stochastic processes, predictive information is not distributed uniformly across past observations; rather, more recent observations often carry greater weight than those further back. The classical \(m\)-th order Markov chain illustrates this phenomenon: for every \(n \ge m\) and every \(a \in \mathcal{A}\),
\begin{equation}
\mathbb{P}(X_{n+1}=a \mid X_1^n) = \mathbb{P}(X_{n+1}=a \mid X_{n-m+1}^n) \quad \text{almost surely},
\label{eq:markov_property_intro}
\end{equation}
where $X_{n-m+1}^n := (X_{n-m+1}, \dots, X_n)$ denotes the block of the most recent $m$ observations (the general definition for arbitrary memory length $k$ is given in \eqref{eq:recent_history_def} below). Thus the most recent $m$ observations contain all predictive information present in the entire history; observations older than \(m\) steps provide no further information beyond what is already contained in the recent \(m\)-block.

This observation suggests a systematic enrichment of the MPE framework: for each memory length $k$, we may consider the predictive gain obtained when only the most recent $k$ observations are retained. This yields a family of quantities $\{\Delta_{n,k}^{\text{mem}}\}_{0 \le k \le n}$ (see Definition~\ref{def:Delta_k} below) that includes the original MPE $\Delta_n^{\text{mem}}$ as the special case $k = n$. We now formalize these ideas through the notion of the finite-history Memory Prediction Excess.

\subsection{Definition of Finite-History Memory Prediction Excess}
\label{subsec:finite_definition}

The MPE was originally defined using the full history \(X_1^n\) as the conditioning information. We now introduce the finite-history counterpart, where only a limited number of past observations is retained. For a fixed memory length \(k\) (\(0 \le k \le n\)), the finite-history MPE (FH-MPE), denoted \(\Delta_{n,k}^{\text{mem}}\) (see Definition~\ref{def:Delta_k}), measures the average improvement in optimal prediction accuracy obtained by conditioning on the most recent \(k\) observations, relative to the same static baseline \(U_n\) used throughout the paper. This family of quantities interpolates between the static baseline at \(k=0\) and the full-history MPE at \(k=n\). Its behavior as a function of \(k\) will be examined in the following subsections: monotonicity in Subsection~\ref{subsec:monotonicity}, saturation for Markov chains in Subsection~\ref{subsec:saturation}, and the related notion of minimal memory length in Subsection~\ref{subsec:effective_memory}.

We begin by fixing notation for the recent history of length $k$. For integers $n \ge 1$ and $0 \le k \le n$, let
\begin{equation}
X_{n-k+1}^n := (X_{n-k+1}, X_{n-k+2}, \dots, X_n)
\label{eq:recent_history_def}
\end{equation}
denote the block of $k$ consecutive observations immediately preceding time $n+1$. This block represents the last $k$ steps of the process. For the boundary case $k=0$, there are no observations to condition upon. To treat this case rigorously within our measure-theoretic framework, we adopt the following convention.

\textbf{Convention (Conditioning on the trivial $\sigma$-algebra for empty history).} 
When $k=0$, the conditioning information is empty. In the measure-theoretic formulation of conditional expectation, this corresponds to conditioning on the trivial $\sigma$-algebra
\begin{equation}
\mathcal{F}_0 := \{\emptyset, \Omega\}.
\label{eq:trivial_sigma}
\end{equation}
For any integrable random variable $Y$ defined on $(\Omega, \mathcal{F}, \mathbb{P})$, the conditional expectation given $\mathcal{F}_0$ is defined as the unique $\mathcal{F}_0$-measurable random variable satisfying
\[
\int_E \mathbb{E}[Y \mid \mathcal{F}_0] \, d\mathbb{P} = \int_E Y \, d\mathbb{P}, \qquad \forall E \in \mathcal{F}_0.
\]
Since $\mathcal{F}_0$ contains only $\emptyset$ and $\Omega$, this definition yields
\begin{equation}
\mathbb{E}[Y \mid \mathcal{F}_0] = \mathbb{E}[Y] \quad \text{almost surely}.
\label{eq:cond_trivial}
\end{equation}
Invoking the standard measure-theoretic identification of conditional probability as the conditional expectation of an indicator function (as used in Step~1 of the proof of Proposition~\ref{prop:MPE_lower_bound}), we obtain, for any event $A \in \mathcal{F}$,
\begin{equation}
\mathbb{P}(A \mid \mathcal{F}_0) = \mathbb{E}[\mathbf{1}_A \mid \mathcal{F}_0] = \mathbb{E}[\mathbf{1}_A] = \mathbb{P}(A) \quad \text{almost surely}.
\label{eq:prob_trivial}
\end{equation}

The trivial $\sigma$-algebra $\mathcal{F}_0 = \{\emptyset, \Omega\}$ is the coarsest $\sigma$-algebra on $\Omega$; the only measurable events are the impossible event $\emptyset$ and the certain event $\Omega$. Conditioning on $\mathcal{F}_0$ therefore amounts to conditioning on no information at all, which is precisely the intended meaning of the empty history. This convention is standard in probability theory and provides a rigorous foundation for the $k=0$ boundary case.

In the context of our finite-history framework, we interpret the empty history $k=0$ as conditioning on $\mathcal{F}_0$. Thus, for each $a \in \mathcal{A}$, the conditional probability $\mathbb{P}(X_{n+1}=a \mid X_{n+1}^n)$ with $k=0$ (so that the index block $X_{n+1}^n$ is empty) is understood as
\[
\mathbb{P}(X_{n+1}=a \mid \mathcal{F}_0) = \mathbb{P}(X_{n+1}=a) \quad \text{almost surely}.
\]

With this convention in place, we may now define the finite-history generalization of the Memory Prediction Excess.

\begin{definition}[Finite-History Memory Prediction Excess (FH-MPE)]\label{def:Delta_k}
For a fixed time index $n \ge 1$ and an integer $k$ satisfying $0 \le k \le n$, let a concrete history block be denoted by $h = (x_{n-k+1}, \dots, x_n) \in \mathcal{A}^k$. For such a block, define the maximal conditional probability given this specific realization as
\[
M_k(h) := \max_{a \in \mathcal{A}} \mathbb{P}(X_{n+1}=a \mid X_{n-k+1}^n = h).
\]
When the history block is random, this quantity becomes a random variable through its dependence on $X_{n-k+1}^n$. More formally, we set
\begin{equation}
M_{n,k}(\omega) := M_k(X_{n-k+1}^n(\omega)) = \max_{a \in \mathcal{A}} \mathbb{P}(X_{n+1}=a \mid X_{n-k+1}^n)(\omega), \qquad \omega \in \Omega.
\label{eq:M_n_k_def}
\end{equation}
By construction, $M_{n,k}$ is a $\sigma(X_{n-k+1}^n)$-measurable random variable. It represents the highest probability of correctly predicting $X_{n+1}$ that can be achieved after observing the most recent $k$ observations $X_{n-k+1}^n$. The conditional probability is interpreted according to the convention above when $k=0$.

The \textbf{finite-history Memory Prediction Excess} (FH-MPE) of the process $\{X_n\}_{n\ge 1}$ at step $n$ with memory length $k$ is then defined as
\begin{equation}
\Delta_{n,k}^{\text{mem}} := \mathbb{E}\bigl[ M_{n,k} \bigr] - U_n,
\label{eq:Delta_k_def}
\end{equation}
where $U_n := \max_{a \in \mathcal{A}} \mathbb{P}(X_{n+1}=a)$ is the static optimal prediction accuracy as defined in \eqref{eq:U_n_def}.

For completeness, we record the two boundary cases explicitly.

\paragraph{Case $\boldsymbol{k=0}$ (empty history).}
By the convention above, conditioning on the empty history reduces to the unconditional marginal distribution:
\[
\mathbb{P}(X_{n+1}=a \mid \mathcal{F}_0) = \mathbb{P}(X_{n+1}=a) \quad \text{almost surely}, \qquad a \in \mathcal{A}.
\]
Consequently,
\[
M_{n,0}(\omega) = \max_{a \in \mathcal{A}} \mathbb{P}(X_{n+1}=a \mid \mathcal{F}_0) = \max_{a \in \mathcal{A}} \mathbb{P}(X_{n+1}=a) = U_n, \qquad \omega \in \Omega,
\]
where the equalities hold pointwise since both sides are constant. Taking expectations yields $\mathbb{E}[M_{n,0}] = U_n$, and therefore
\begin{equation}
\Delta_{n,0}^{\text{mem}} = U_n - U_n = 0.
\label{eq:Delta_k_boundary_zero}
\end{equation}

\paragraph{Case $\boldsymbol{k=n}$ (full history).}
The recent history $X_{1}^n$ coincides with the full history $X_1^n$ as defined in Section~\ref{sec:definition}. Hence
\[
M_{n,n}(\omega) = \max_{a \in \mathcal{A}} \mathbb{P}(X_{n+1}=a \mid X_1^n)(\omega) = M_n(\omega), \qquad \omega \in \Omega,
\]
where $M_n$ is the random variable introduced in \eqref{eq:M_n_def}. Taking expectations and subtracting $U_n$ yields
\begin{equation}
\Delta_{n,n}^{\text{mem}} = \mathbb{E}[M_n] - U_n = \Delta_n^{\text{mem}},
\label{eq:Delta_k_boundary_full}
\end{equation}
where the last equality follows from Definition~\ref{def:MPE}.
\end{definition}

Thus the family $\{\Delta_{n,k}^{\text{mem}}\}_{0 \le k \le n}$ interpolates between the static prediction baseline at $k=0$ and the full-history MPE at $k=n$. For intermediate values $0 < k < n$, the quantity $\Delta_{n,k}^{\text{mem}}$ measures the average improvement in optimal prediction accuracy when only the last $k$ observations are retained, with earlier observations discarded.

\begin{remark}[Operational Interpretation of the Memory Length]\label{rem:k_interpretation}
The parameter $k$ admits a clear operational interpretation: it specifies the number of past observations that the predictor is permitted to use. When $k=0$, the predictor operates without any historical information, relying solely on the marginal distribution of $X_{n+1}$. When $k=n$, the predictor has access to the entire observed history from time $1$ to $n$. For $0 < k < n$, the predictor employs a sliding window of length $k$ and disregards observations preceding time $n-k+1$. This parametrization enables a systematic investigation of the marginal contribution of each additional unit of memory to the predictive accuracy, revealing how the predictive value of historical information accumulates as the memory length grows—a phenomenon formally established in Theorem~\ref{thm:Delta_k_monotonicity}. In addition, this convention is consistent with the notation $X_1^n$ used throughout Section~\ref{sec:definition} for the full history, and the trivial $\sigma$-algebra convention avoids the ambiguity that would otherwise arise from expressions such as $X_{n+1}^n$, where the lower index would exceed the upper index.
\end{remark}

\subsection{Monotonicity of FH-MPE}
\label{subsec:monotonicity}

With the finite-history Memory Prediction Excess $\Delta_{n,k}^{\text{mem}}$ (FH-MPE) introduced in Definition~\ref{def:Delta_k}, we now investigate its behavior as the memory length $k$ varies. A reasonable expectation for any prediction system is that providing the predictor with additional past observations should not impair its optimal performance. The following theorem confirms this by establishing that $\Delta_{n,k}^{\text{mem}}$ is nondecreasing in $k$.

\begin{theorem}[Monotonicity of FH-MPE ($\boldsymbol{\Delta_{n,k}^{\text{mem}}}$) with respect to memory length]\label{thm:Delta_k_monotonicity}
For every fixed time index $n \ge 1$ and every integer $k$ satisfying $0 \le k \le n-1$,
\begin{equation}
\Delta_{n,k}^{\text{mem}} \le \Delta_{n,k+1}^{\text{mem}}.
\label{eq:monotonicity_ineq}
\end{equation}
\end{theorem}

\begin{proof}
We prove the inequality $\Delta_{n,k}^{\text{mem}} \le \Delta_{n,k+1}^{\text{mem}}$ by comparing the conditional distributions arising from the two information sets. The argument proceeds in four steps.

\paragraph{Step 1: Information sets and their relation.}
Let
\begin{equation}
\mathcal{G}_k := \sigma(X_{n-k+1}^n), \qquad \mathcal{G}_{k+1} := \sigma(X_{n-k}^n)
\label{eq:sigma_algebras}
\end{equation}
be the $\sigma$-algebras generated by the most recent $k$ and $k+1$ observations, respectively. Since $X_{n-k+1}^n$ is obtained from $X_{n-k}^n$ by deleting the earliest component $X_{n-k}$, we have $\mathcal{G}_k \subseteq \mathcal{G}_{k+1}$; thus $\mathcal{G}_k$ represents coarser information than $\mathcal{G}_{k+1}$.

For each $a \in \mathcal{A}$, define the conditional probabilities
\[
Q_a^{(k)} := \mathbb{P}(X_{n+1}=a \mid \mathcal{G}_k), \qquad
Q_a^{(k+1)} := \mathbb{P}(X_{n+1}=a \mid \mathcal{G}_{k+1}).
\]
Since $\mathcal{G}_k \subseteq \mathcal{G}_{k+1}$, the tower property of conditional expectations applies: for any integrable random variable $Y$,
\[
\mathbb{E}\bigl[ \mathbb{E}[Y \mid \mathcal{G}_{k+1}] \mid \mathcal{G}_k \bigr]
= \mathbb{E}[Y \mid \mathcal{G}_k].
\]
Taking $Y = \mathbf{1}_{\{X_{n+1}=a\}}$ and noting that
\[
\mathbb{E}\bigl[ \mathbf{1}_{\{X_{n+1}=a\}} \mid \mathcal{G}_{k+1} \bigr]
= \mathbb{P}(X_{n+1}=a \mid \mathcal{G}_{k+1}) = Q_a^{(k+1)},
\]
we obtain
\begin{equation}
Q_a^{(k)}
= \mathbb{E}\bigl[ \mathbf{1}_{\{X_{n+1}=a\}} \mid \mathcal{G}_k \bigr]
= \mathbb{E}\bigl[ Q_a^{(k+1)} \mid \mathcal{G}_k \bigr] \quad \text{almost surely}, \qquad a \in \mathcal{A}.
\label{eq:tower_relation}
\end{equation}
This identity expresses the coarser conditional probability $Q_a^{(k)}$ as the conditional expectation (with respect to the coarser $\sigma$-algebra $\mathcal{G}_k$) of the finer conditional probability $Q_a^{(k+1)}$. It will be used explicitly in Step 4 when we take expectations. We shall now make this relation more explicit by decomposing the coarser conditioning event into a weighted average over the possible values of the additional observation $X_{n-k}$.

\paragraph{Step 2: Decomposition of the coarser conditional distribution into a weighted average.}
Since $\mathcal{A}$ is finite, the $\sigma$-algebra $\mathcal{G}_k$ is generated by the partition of $\Omega$ into atoms of the form
\[
A_h := \{\omega \in \Omega : X_{n-k+1}^n(\omega) = h\}, \qquad h \in \mathcal{A}^k.
\]
Each atom $A_h$ corresponds to a specific realization $h$ of the most recent $k$ observations.

Fix a history $h \in \mathcal{A}^k$ with $\mathbb{P}(A_h) > 0$. For each $x \in \mathcal{A}$, define the conditional weight
\[
w_{h,x} := \mathbb{P}(X_{n-k}=x \mid X_{n-k+1}^n = h),
\]
which is well-defined by the positive-probability assumption (note that since $k \le n-1$ in the present context, we have $n-k \ge 1$, so $X_{n-k}$ is a well-defined random variable). These weights are nonnegative and, by the definition of conditional probability, satisfy
\begin{equation}
\sum_{x \in \mathcal{A}} w_{h,x}
= \sum_{x \in \mathcal{A}} \mathbb{P}(X_{n-k}=x \mid X_{n-k+1}^n = h)
= 1.
\label{eq:weights_normalization}
\end{equation}
For histories with zero probability (i.e., $\mathbb{P}(X_{n-k+1}^n = h) = 0$), the weights may be defined arbitrarily (e.g., uniformly over $\mathcal{A}$), as such histories occur on a null set and therefore do not affect any expectation or almost-sure statement; this convention is consistent with the treatment of null events in Section~\ref{sec:definition}. For such null histories, the decompositions below are irrelevant and need not be verified.

We now derive the key decomposition. For each $a \in \mathcal{A}$ and each positive-probability history $h$ (i.e., $\mathbb{P}(X_{n-k+1}^n = h) > 0$), begin with the conditional probability $\mathbb{P}(X_{n+1}=a \mid X_{n-k+1}^n = h)$. By the definition of conditional probability, this is equal to $\mathbb{P}(X_{n+1}=a,\; X_{n-k+1}^n = h) / \mathbb{P}(X_{n-k+1}^n = h)$.

The numerator is the probability of the joint event $B_h := \{X_{n+1}=a,\; X_{n-k+1}^n = h\}$. Since the events $\{X_{n-k}=x\}$ for $x \in \mathcal{A}$ form a partition of $\Omega$ (up to a null set), we have
\[
B_h = B_h \cap \Omega = B_h \cap \left( \bigcup_{x \in \mathcal{A}} \{X_{n-k}=x\} \right)
= \bigcup_{x \in \mathcal{A}} \left( B_h \cap \{X_{n-k}=x\} \right).
\]
Thus this joint event can be decomposed into the disjoint union
\[
\{X_{n+1}=a,\; X_{n-k+1}^n = h\}
= \bigcup_{x \in \mathcal{A}} \{X_{n+1}=a,\; X_{n-k+1}^n = h,\; X_{n-k}=x\}.
\]
Then, by finite additivity of probability measures,
\begin{equation}
\mathbb{P}(X_{n+1}=a,\; X_{n-k+1}^n = h)
= \sum_{x \in \mathcal{A}} \mathbb{P}(X_{n+1}=a,\; X_{n-k+1}^n = h,\; X_{n-k}=x).
\label{eq:joint_decomposition}
\end{equation}

For each $x \in \mathcal{A}$, consider the summand in \eqref{eq:joint_decomposition}. If $w_{h,x}=0$, then $\mathbb{P}(X_{n-k+1}^n = h,\; X_{n-k}=x) = 0$, and consequently the summand itself is zero, contributing nothing to the sum. Such terms may therefore be omitted. For $w_{h,x}>0$, we have $\mathbb{P}(X_{n-k+1}^n = h,\; X_{n-k}=x) = \mathbb{P}(X_{n-k+1}^n = h) \cdot w_{h,x} > 0$, so the conditional probability
\[
\mathbb{P}(X_{n+1}=a \mid X_{n-k+1}^n = h,\; X_{n-k}=x)
\]
is well-defined. Applying the conditional probability multiplication formula (chain rule) to the three events
\[
E_1 := \{X_{n-k+1}^n = h\},\qquad
E_2 := \{X_{n-k}=x\},\qquad
E_3 := \{X_{n+1}=a\},
\]
we obtain
\[
\mathbb{P}(E_1 \cap E_2 \cap E_3)
= \mathbb{P}(E_1) \cdot \mathbb{P}(E_2 \mid E_1) \cdot \mathbb{P}(E_3 \mid E_1 \cap E_2).
\]
Substituting the events back yields
\[
\mathbb{P}(X_{n+1}=a,\; X_{n-k+1}^n = h,\; X_{n-k}=x)
= \mathbb{P}(X_{n-k+1}^n = h) \cdot w_{h,x} \cdot \mathbb{P}(X_{n+1}=a \mid X_{n-k+1}^n = h,\; X_{n-k}=x).
\]
Substituting this expression into \eqref{eq:joint_decomposition} and canceling the common positive factor $\mathbb{P}(X_{n-k+1}^n = h)$ from numerator and denominator yields, for each $a \in \mathcal{A}$ and each positive-probability history $h$,
\begin{align}
\mathbb{P}(X_{n+1}=a \mid X_{n-k+1}^n = h)
&= \sum_{x \in \mathcal{A}} w_{h,x} \cdot \mathbb{P}(X_{n+1}=a \mid X_{n-k+1}^n = h,\; X_{n-k}=x),
\label{eq:total_prob_decomp}
\end{align}
where the conditioning event in the last factor is equivalent to $\{X_{n-k}^n = (x,h)\}$, with $(x,h)$ denoting the concatenation of $x$ with the block $h$ (i.e., the finer history of length $k+1$). Note that terms with $w_{h,x}=0$ contribute zero to the sum in \eqref{eq:total_prob_decomp} and may thus be omitted; the remaining terms are precisely those for which the conditional probabilities are well-defined.

This decomposition expresses the conditional distribution given the coarser history $h$ (consisting of $k$ observations) as a weighted average, with weights $w_{h,x}$ satisfying \eqref{eq:weights_normalization}, of the conditional distributions given the finer history $(x,h)$ (consisting of $k+1$ observations). The weight $w_{h,x}$ represents the probability that the additional observation $X_{n-k}$ takes the value $x$, given the coarser history $h$. The normalization \eqref{eq:weights_normalization} ensures that the right-hand side of \eqref{eq:total_prob_decomp} is indeed a convex combination of the finer conditional probabilities.

\paragraph{Step 3: Pointwise comparison of the maxima.}
Fix $h \in \mathcal{A}^k$ with $\mathbb{P}(X_{n-k+1}^n = h) > 0$. For each $x \in \mathcal{A}$ and each $a \in \mathcal{A}$, we have the elementary inequality
\[
\mathbb{P}(X_{n+1}=a \mid X_{n-k+1}^n = h,\; X_{n-k}=x)
\le \max_{b \in \mathcal{A}} \mathbb{P}(X_{n+1}=b \mid X_{n-k+1}^n = h,\; X_{n-k}=x).
\]
For each fixed $x \in \mathcal{A}$, define
\[
M_{k+1}((x,h)) := \max_{b \in \mathcal{A}} \mathbb{P}(X_{n+1}=b \mid X_{n-k+1}^n = h,\; X_{n-k}=x),
\]
which is the optimal conditional prediction accuracy given the complete finer history $(x,h)$ of length $k+1$ (see Step 2 for the concatenation convention, where $(x,h)$ denotes the finer block of length $k+1$).

Substituting this definition into the elementary inequality—i.e., replacing the right-hand side of the inequality by $M_{k+1}((x,h))$ for each fixed $x$—and applying the resulting inequality termwise to the decomposition \eqref{eq:total_prob_decomp}, while noting that the weights $w_{h,x}$ are nonnegative, yields, for each $a \in \mathcal{A}$,
\[
\mathbb{P}(X_{n+1}=a \mid X_{n-k+1}^n = h)
\le \sum_{x \in \mathcal{A}} w_{h,x} \cdot \max_{b \in \mathcal{A}} \mathbb{P}(X_{n+1}=b \mid X_{n-k+1}^n = h,\; X_{n-k}=x)
= \sum_{x \in \mathcal{A}} w_{h,x} \cdot M_{k+1}((x,h)).
\]
The right-hand side of \eqref{eq:pointwise_bound} does not depend on $a$. Therefore, taking the maximum over $a \in \mathcal{A}$ on the left-hand side preserves the inequality, yielding
\begin{equation}
\max_{a \in \mathcal{A}} \mathbb{P}(X_{n+1}=a \mid X_{n-k+1}^n = h)
\le \sum_{x \in \mathcal{A}} w_{h,x} \cdot M_{k+1}((x,h)).
\label{eq:max_pointwise_bound}
\end{equation}
The left-hand side is precisely $M_k(h)$ as defined in Definition~\ref{def:Delta_k}.

It remains to interpret the right-hand side of \eqref{eq:max_pointwise_bound}. Since the events $\{X_{n-k}^n = (x,h)\}$ for $x \in \mathcal{A}$ form a partition (up to a null set) of the event $\{X_{n-k+1}^n = h\}$, where $(x,h) := (x, h_1, \dots, h_k)$ denotes the finer history of length $k+1$ obtained by prepending the additional observation $x$ to the block $h = (h_1, \dots, h_k)$, the conditional expectation on the right-hand side can be expanded as follows.

Fix a positive-probability history $h$ (so that $\mathbb{P}(X_{n-k+1}^n = h) > 0$). Recall that, for a $\sigma$-algebra $\mathcal{G} \subseteq \mathcal{F}$, the conditional expectation $\mathbb{E}[Y \mid \mathcal{G}]$ is the unique (up to almost-sure equality) $\mathcal{G}$-measurable random variable $Z$ satisfying
\[
\int_G Z \, d\mathbb{P} = \int_G Y \, d\mathbb{P}, \qquad \forall G \in \mathcal{G}.
\]
In the special case where $\mathcal{G} = \sigma(B)$ for an event $B$ with $\mathbb{P}(B) > 0$, the conditional expectation $\mathbb{E}[Y \mid B]$ (often written as $\mathbb{E}[Y \mid \mathbf{1}_B]$) reduces to the elementary formula
\[
\mathbb{E}[Y \mid B] = \frac{\mathbb{E}\bigl[ Y \cdot \mathbf{1}_B \bigr]}{\mathbb{P}(B)}.
\]
This is obtained by taking $G = B$ in the defining identity and using the fact that $\mathbb{E}[Y \mid \sigma(B)]$ is constant on $B$ and on $B^c$. Applying this elementary formula with $B = \{X_{n-k+1}^n = h\}$ and $Y = M_{n,k+1}$ (as defined in \eqref{eq:M_n_k_def}) gives
\begin{equation}
\mathbb{E}\bigl[ M_{n,k+1} \mid X_{n-k+1}^n = h \bigr]
= \frac{\mathbb{E}\bigl[ M_{n,k+1} \cdot \mathbf{1}_{\{X_{n-k+1}^n = h\}} \bigr]}
        {\mathbb{P}(X_{n-k+1}^n = h)}.
\label{eq:cond_exp_event_form}
\end{equation}

The event $\{X_{n-k+1}^n = h\}$ is partitioned into the disjoint union
\[
\{X_{n-k+1}^n = h\}
= \bigcup_{x \in \mathcal{A}} \{X_{n-k}^n = (x,h)\},
\]
where $(x,h)$ denotes the concatenation of $x$ with the block $h$, i.e., the finer history of length $k+1$ as in Step~2. Consequently, by the linearity of expectation,
\begin{equation}
\mathbb{E}\bigl[ M_{n,k+1} \cdot \mathbf{1}_{\{X_{n-k+1}^n = h\}} \bigr]
= \sum_{x \in \mathcal{A}} \mathbb{E}\bigl[ M_{n,k+1} \cdot \mathbf{1}_{\{X_{n-k}^n = (x,h)\}} \bigr].
\label{eq:expectation_partition}
\end{equation}
On each atom $\{X_{n-k}^n = (x,h)\}$, the random variable $M_{n,k+1}$ takes the constant value $M_{k+1}((x,h))$: indeed, $M_{n,k+1} = M_{k+1}(X_{n-k}^n)$ by Definition~\ref{def:Delta_k}, and on this atom the argument $X_{n-k}^n$ is identically equal to $(x,h)$. Therefore,
\[
\mathbb{E}\bigl[ M_{n,k+1} \cdot \mathbf{1}_{\{X_{n-k}^n = (x,h)\}} \bigr]
= M_{k+1}((x,h)) \cdot \mathbb{P}(X_{n-k}^n = (x,h)).
\]
Substituting this back into \eqref{eq:expectation_partition} and then into \eqref{eq:cond_exp_event_form} yields
\[
\mathbb{E}\bigl[ M_{n,k+1} \mid X_{n-k+1}^n = h \bigr]
= \sum_{x \in \mathcal{A}} M_{k+1}((x,h)) \cdot
\frac{\mathbb{P}(X_{n-k}^n = (x,h))}{\mathbb{P}(X_{n-k+1}^n = h)}.
\]
By the definition of conditional probability, the ratio is precisely
\[
\frac{\mathbb{P}(X_{n-k}=x,\; X_{n-k+1}^n = h)}
        {\mathbb{P}(X_{n-k+1}^n = h)}
= \mathbb{P}(X_{n-k}=x \mid X_{n-k+1}^n = h)
= w_{h,x}.
\]
Hence we obtain
\[
\mathbb{E}\bigl[ M_{n,k+1} \mid X_{n-k+1}^n = h \bigr]
= \sum_{x \in \mathcal{A}} w_{h,x} \cdot M_{k+1}((x,h)).
\]

Thus \eqref{eq:max_pointwise_bound} may be rewritten as
\begin{equation}
M_k(h) \le \mathbb{E}\bigl[ M_{n,k+1} \mid X_{n-k+1}^n = h \bigr]
\label{eq:pointwise_cond_ineq}
\end{equation}
for every $h$ with positive probability. For histories with zero probability, the inequality need not be verified, as such histories occur on a null set and therefore do not affect any expectation or almost-sure statement.

\paragraph{Step 4: Taking expectations.}
We now convert the pointwise inequality established in Step~3 into a global expectation inequality. Recall that \eqref{eq:pointwise_cond_ineq} holds for every positive-probability history $h$; for zero-probability histories the inequality need not be verified, as such histories occur on a null set and therefore do not affect any expectation.

Since \eqref{eq:pointwise_cond_ineq} holds for every positive-probability history $h$, and the collection of zero-probability histories has total probability zero, we may replace the fixed value $h$ by the random history block $X_{n-k+1}^n(\omega)$. To see this, for any fixed $\omega$, setting $h = X_{n-k+1}^n(\omega)$: if $\mathbb{P}(X_{n-k+1}^n = h) > 0$, the inequality applies directly; if $\mathbb{P}(X_{n-k+1}^n = h) = 0$, such $\omega$'s form a null set and can be ignored.

The inequality
\[
M_k(h) \le \mathbb{E}\bigl[ M_{n,k+1} \mid X_{n-k+1}^n = h \bigr]
\]
holds for all $h$ with positive probability. Since the random variable $\mathbb{E}[M_{n,k+1} \mid X_{n-k+1}^n]$ takes the value $\mathbb{E}[M_{n,k+1} \mid X_{n-k+1}^n = h]$ on the atom $\{X_{n-k+1}^n = h\}$, substituting $h = X_{n-k+1}^n(\omega)$ yields, for almost every $\omega \in \Omega$,
\[
M_{n,k}(\omega) \le \mathbb{E}\bigl[ M_{n,k+1} \mid X_{n-k+1}^n \bigr](\omega),
\]
where the left-hand side follows from \(M_{n,k} = M_k(X_{n-k+1}^n)\) (see \eqref{eq:M_n_k_def}). Hence
\[
M_{n,k} \le \mathbb{E}\bigl[ M_{n,k+1} \mid X_{n-k+1}^n \bigr] \quad \text{almost surely}.
\]
Taking expectations on both sides and applying the tower property of conditional expectations (which gives \(\mathbb{E}[\mathbb{E}[M_{n,k+1} \mid X_{n-k+1}^n]] = \mathbb{E}[M_{n,k+1}]\)) yields
\begin{equation}
\mathbb{E}[M_{n,k}] \le \mathbb{E}[M_{n,k+1}].
\label{eq:expectation_comparison}
\end{equation}

Subtracting the constant $U_n$ from both sides of \eqref{eq:expectation_comparison} gives
\[
\mathbb{E}[M_{n,k}] - U_n \le \mathbb{E}[M_{n,k+1}] - U_n.
\]
By Definition~\ref{def:Delta_k}, the left-hand side is $\Delta_{n,k}^{\text{mem}}$ and the right-hand side is $\Delta_{n,k+1}^{\text{mem}}$. Therefore,
\[
\Delta_{n,k}^{\text{mem}} \le \Delta_{n,k+1}^{\text{mem}},
\]
which establishes \eqref{eq:monotonicity_ineq}. The proof is complete.
\end{proof}

The following corollary records the full chain of inequalities obtained by iterating the monotonicity result.

\begin{corollary}[Monotonicity chain for FH-MPE ($\boldsymbol{\Delta_{n,k}^{\text{mem}}}$)]\label{cor:Delta_k_chain}
For every fixed $n \ge 1$,
\begin{equation}
0 = \Delta_{n,0}^{\text{mem}} \le \Delta_{n,1}^{\text{mem}} \le \cdots \le \Delta_{n,n}^{\text{mem}} = \Delta_n^{\text{mem}}.
\label{eq:monotonicity_chain}
\end{equation}
In particular, the full-history Memory Prediction Excess $\Delta_n^{\text{mem}}$ serves as the universal upper bound for all finite-history FH-MPEs at the same time index $n$.
\end{corollary}

\begin{proof}
We apply Theorem~\ref{thm:Delta_k_monotonicity} repeatedly to compare the FH-MPE values for consecutive memory lengths.

For each integer $j$ in the range $0 \le j \le n-1$, where Theorem~\ref{thm:Delta_k_monotonicity} applies, we have the inequality $\Delta_{n,j}^{\text{mem}} \le \Delta_{n,j+1}^{\text{mem}}$. Applying this successively for $j = 0, 1, \dots, n-1$ yields the chain
\[
\Delta_{n,0}^{\text{mem}} \le \Delta_{n,1}^{\text{mem}} \le \Delta_{n,2}^{\text{mem}} \le \cdots \le \Delta_{n,n-1}^{\text{mem}} \le \Delta_{n,n}^{\text{mem}}.
\]
It remains to identify the two endpoints of this chain.

The first term $\Delta_{n,0}^{\text{mem}}$ corresponds to the case $k=0$, i.e., conditioning on the empty history. By \eqref{eq:Delta_k_boundary_zero}, we have $\Delta_{n,0}^{\text{mem}} = 0$; this follows from the convention, established in Definition~\ref{def:Delta_k}, that conditioning on the empty history reduces to the unconditional marginal distribution, so that $M_{n,0} = U_n$ pointwise and hence $\mathbb{E}[M_{n,0}] = U_n$.

Similarly, the last term $\Delta_{n,n}^{\text{mem}}$ corresponds to the case $k=n$, i.e., conditioning on the full history $X_1^n$. By \eqref{eq:Delta_k_boundary_full}, we have $\Delta_{n,n}^{\text{mem}} = \Delta_n^{\text{mem}}$; indeed, when $k=n$, the recent history $X_{1}^n$ coincides with the full history $X_1^n$ as defined in Section~\ref{sec:definition}, so $M_{n,n} = M_n$ and consequently $\Delta_{n,n}^{\text{mem}} = \Delta_n^{\text{mem}}$.

Substituting these two endpoint identities into the chain above, we obtain
\[
0 = \Delta_{n,0}^{\text{mem}} \le \Delta_{n,1}^{\text{mem}} \le \Delta_{n,2}^{\text{mem}} \le \cdots \le \Delta_{n,n-1}^{\text{mem}} \le \Delta_{n,n}^{\text{mem}} = \Delta_n^{\text{mem}},
\]
which is precisely \eqref{eq:monotonicity_chain}. Thus, for every fixed $n$, the full-history MPE $\Delta_n^{\text{mem}}$ serves as the universal upper bound for all finite-history FH-MPEs $\Delta_{n,k}^{\text{mem}}$ with $0 \le k \le n$.
\end{proof}

\begin{remark}[Interpretation of non-negativity and the monotonicity chain]\label{rem:monotonicity_chain_interpretation}
We observe that the monotonicity chain in Corollary~\ref{cor:Delta_k_chain}, which follows directly from Theorem~\ref{thm:Delta_k_monotonicity}, can be understood from two perspectives. Taken as a whole, the chain of inequalities
\[
0 = \Delta_{n,0}^{\text{mem}} \le \Delta_{n,1}^{\text{mem}} \le \cdots \le \Delta_{n,n}^{\text{mem}} = \Delta_n^{\text{mem}}
\]
effectively traces a complete spectrum ranging from ``no memory'' to ``full memory.''

First, the monotonicity chain reaffirms the non-negativity of all finite-history MPEs: for every $0 \le k \le n$, we have $\Delta_{n,k}^{\text{mem}} \ge 0$, with equality at $k=0$. This means that even when the predictor uses only a limited portion of the past, its optimal prediction accuracy is, on average, no worse than that based solely on the static marginal distribution. In this sense, the non-negativity property established for the full-history case in Theorem~\ref{thm:nonneg} is naturally inherited and extended to the finite-history setting.

Second, the chain reveals how this non-negative advantage accumulates monotonically as the memory length increases: as $k$ ranges from $0$ to $n$, the values $\Delta_{n,k}^{\text{mem}}$ increase monotonically, eventually reaching the full-history MPE $\Delta_n^{\text{mem}}$ at $k=n$. The full-history MPE thereby serves as a universal upper bound for all intermediate states, confirming that the complete past indeed provides the maximal possible predictive gain. In this way, the contribution of each additional unit of historical information to predictive accuracy is made explicit.

The monotonicity chain, which we have established through a detailed analysis of the conditional probability structure and its evolution with respect to the memory length, in fact does not require any structural assumptions beyond the discrete-time finite-state setting. In this sense, it may be viewed as a basic framework for understanding the memory-dependence structure of stochastic processes: within this setting, the predictive value of memory accumulates monotonically with the length of the available history, as established by the monotonicity chain, with the full-history MPE serving as the natural upper bound of this accumulation.
\end{remark}

\subsection{Saturation of FH-MPE for Finite-Order Markov Chains}
\label{subsec:saturation}

In the previous subsections we introduced the FH-MPE and proved its monotonicity with respect to $k$ (Theorem~\ref{thm:Delta_k_monotonicity}). For general stochastic processes, this monotone increase may continue throughout the entire range $0 \le k \le n$, with the possibility that additional historical information yields further predictive gain. An important class of processes for which the predictive value of additional historical information saturates once the memory length reaches a finite order is provided by the classical notion of a Markov chain of finite order. This class of processes, along with its various extensions and ramifications, has been investigated in a number of contexts \cite{misiewicz2026markov, lohr2009generative, venegas2023optimality}. In this subsection we apply the FH-MPE framework to this class and show that for such processes, the FH-MPE attains its full-history value once the memory length reaches the Markov order (Theorem~\ref{thm:Delta_k_saturation}); moreover, combining this saturation with the monotonicity chain yields a complete characterization of the evolution of the FH-MPE for Markov chains (Corollary~\ref{cor:saturation_chain}).

We begin by recalling the definition of a finite-order Markov chain.

\begin{definition}[$\boldsymbol{m}$-th order Markov chain]\label{def:m_markov}
Let $m \ge 1$ be a fixed integer. A discrete-time stochastic process $\{X_n\}_{n\ge 1}$ taking values in the finite state space $\mathcal{A}$ is called an $m$-th order Markov chain if, for every $n \ge m$ and every $a \in \mathcal{A}$,
\begin{equation}
\mathbb{P}(X_{n+1}=a \mid X_1^n) = \mathbb{P}(X_{n+1}=a \mid X_{n-m+1}^n) \quad \text{almost surely}.
\label{eq:m_markov_property}
\end{equation}
In words, given the most recent $m$ observations, the future state $X_{n+1}$ is conditionally independent of all earlier observations; equivalently, the entire history $X_1^n$ contains no additional predictive information beyond that already provided by the recent $m$-block $X_{n-m+1}^n$ (as defined in \eqref{eq:recent_history_def} with $k=m$).
\end{definition}

It is important to note that Definition~\ref{def:m_markov} is stated for all $n \ge m$. The condition $n \ge m$ ensures that the recent $m$-block $X_{n-m+1}^n$ is well-defined; for $n < m$, there are not yet $m$ observations available to form such a block. Throughout this subsection we will exclusively consider time indices $n$ satisfying $n \ge m$, so that the Markov property is applicable. Also note that when the conditioning event $\{X_{n-m+1}^n = h\}$ has zero probability, the conditional probability in \eqref{eq:m_markov_property} is not defined by the elementary ratio formula; however, the almost-sure formulation in terms of $\sigma$-algebras is well-defined and suffices for our purposes, consistent with the treatment of null events in Section~\ref{sec:definition}.

The following theorem establishes that for an $m$-th order Markov chain, the FH-MPE saturates once the memory length $k$ (see \eqref{eq:recent_history_def}) reaches the order $m$.

\begin{theorem}[Saturation of FH-MPE for Markov chains]\label{thm:Delta_k_saturation}
Let $\{X_n\}_{n\ge 1}$ be an $m$-th order Markov chain (see Definition~\ref{def:m_markov}) for some fixed integer $m \ge 1$. Then for every $n \ge m$ and every integer $k$ satisfying $m \le k \le n$,
\begin{equation}
\Delta_{n,k}^{\text{mem}} = \Delta_{n,m}^{\text{mem}}.
\label{eq:saturation_equality}
\end{equation}
Consequently, for all $k \ge m$ (with $k \le n$), the FH-MPE remains constant at the value $\Delta_{n,m}^{\text{mem}}$; no additional predictive gain is obtained by extending the memory beyond the Markov order $m$.
\end{theorem}

\begin{proof}
We prove the theorem by showing that the conditional distributions used in the definitions of $M_{n,k}$ coincide for all $k \ge m$. The argument proceeds in three steps.

\paragraph{Step 1: Establishing conditional independence of the extra older observations.}
Fix $n \ge m$ and let $k$ be an integer satisfying $m \le k \le n$. Consider the two information sets: the most recent $k$ observations $X_{n-k+1}^n$ and the most recent $m$ observations $X_{n-m+1}^n$. Since $k \ge m$, we have $n-k+1 \le n-m+1$, so that the block $X_{n-m+1}^n$ is contained as a contiguous sub-block within $X_{n-k+1}^n$. The additional observations in the longer history, beyond those in the $m$-block, are precisely
\[
X_{n-k+1}^{n-m} := (X_{n-k+1}, X_{n-k+2}, \dots, X_{n-m}),
\]
the block of observations older than $n-m+1$ but still within the $k$-window. If $k=m$, this additional block is empty, and the argument below applies with the standard convention that conditioning on an empty history reduces to the unconditional distribution.

By the $m$-th order Markov property \eqref{eq:m_markov_property}, for every $a \in \mathcal{A}$,
\[
\mathbb{P}(X_{n+1}=a \mid X_1^n) = \mathbb{P}(X_{n+1}=a \mid X_{n-m+1}^n) \quad \text{almost surely}.
\]
This property asserts that, conditional on $X_{n-m+1}^n$, the future $X_{n+1}$ is independent of the entire history $X_1^{n-m}$ (i.e., all observations earlier than time $n-m+1$). Since the block $X_{n-k+1}^{n-m}$ is a sub-block of $X_1^{n-m}$ (whenever $k>m$; if $k=m$, the block is empty and the independence is trivial), it follows from the definition of conditional independence that $X_{n+1}$ is also conditionally independent of this shorter block given $X_{n-m+1}^n$. More precisely, for every $a \in \mathcal{A}$,
\begin{equation}
\mathbb{P}(X_{n+1}=a \mid X_{n-m+1}^n,\; X_{n-k+1}^{n-m}) = \mathbb{P}(X_{n+1}=a \mid X_{n-m+1}^n) \quad \text{almost surely}.
\label{eq:cond_independence}
\end{equation}
Since the full $k$-block $X_{n-k+1}^n$ is obtained by concatenating $X_{n-k+1}^{n-m}$ and $X_{n-m+1}^n$, conditioning on $X_{n-k+1}^n$ is equivalent to conditioning on the pair $(X_{n-k+1}^{n-m}, X_{n-m+1}^n)$. Thus \eqref{eq:cond_independence} expresses precisely the fact that the extra older observations $X_{n-k+1}^{n-m}$ carry no additional predictive information about $X_{n+1}$ beyond that already contained in $X_{n-m+1}^n$.

\paragraph{Step 2: Identifying the conditional distributions for the two history lengths.}
From \eqref{eq:cond_independence}, it follows immediately that for every $a \in \mathcal{A}$,
\begin{equation}
\mathbb{P}(X_{n+1}=a \mid X_{n-k+1}^n) = \mathbb{P}(X_{n+1}=a \mid X_{n-m+1}^n) \quad \text{almost surely}.
\label{eq:cond_distribution_equality}
\end{equation}
Indeed, the left-hand side is $\mathbb{P}(X_{n+1}=a \mid X_{n-m+1}^n,\; X_{n-k+1}^{n-m})$, which by conditional independence equals $\mathbb{P}(X_{n+1}=a \mid X_{n-m+1}^n)$, the right-hand side. The equality holds for every $a \in \mathcal{A}$, with the almost-sure qualifier accounting for the possibility of conditioning on zero-probability events; the $\sigma$-algebra formulation ensures that this statement is well-defined.

Since \eqref{eq:cond_distribution_equality} holds for every $a \in \mathcal{A}$, the two conditional distributions on $\mathcal{A}$ coincide almost surely. More explicitly, for each $a \in \mathcal{A}$, there exists a set $\Omega_a \subseteq \Omega$ with $\mathbb{P}(\Omega_a) = 1$ such that the equality of conditional probabilities holds on $\Omega_a$. Since $\mathcal{A}$ is finite, the intersection $\Omega_0 := \bigcap_{a \in \mathcal{A}} \Omega_a$ also has $\mathbb{P}(\Omega_0) = 1$; on this common full-measure set, the equality holds simultaneously for all $a \in \mathcal{A}$. Thus, for almost every $\omega \in \Omega$, the two probability vectors
\[
\bigl(\mathbb{P}(X_{n+1}=a \mid X_{n-k+1}^n)(\omega)\bigr)_{a \in \mathcal{A}}
\quad\text{and}\quad
\bigl(\mathbb{P}(X_{n+1}=a \mid X_{n-m+1}^n)(\omega)\bigr)_{a \in \mathcal{A}}
\]
are identical. Taking the maximum over the finite set $\mathcal{A}$ on both sides—which is legitimate since the equality of vectors holds for every $\omega \in \Omega_0$, and the maximum of two identical finite vectors is equal—we obtain
\[
\max_{a \in \mathcal{A}} \mathbb{P}(X_{n+1}=a \mid X_{n-k+1}^n) = \max_{a \in \mathcal{A}} \mathbb{P}(X_{n+1}=a \mid X_{n-m+1}^n) \quad \text{almost surely}.
\]
In the notation of Definition~\ref{def:Delta_k}, this is precisely
\begin{equation}
M_{n,k} = M_{n,m} \quad \text{almost surely}.
\label{eq:M_n_k_equal}
\end{equation}

This equality is an equality of random variables: for almost every $\omega \in \Omega$, the value of the optimal conditional prediction accuracy based on the $k$-block is identical to that based on the $m$-block. The extra observations in the $k$-block do not alter the maximum of the conditional distribution.

\paragraph{Step 3: Taking expectations and concluding the proof.}
Taking expectations on both sides of \eqref{eq:M_n_k_equal}, we obtain
\[
\mathbb{E}[M_{n,k}] = \mathbb{E}[M_{n,m}],
\]
since the expectation of two almost-surely equal random variables is identical (their difference is zero almost surely, hence has expectation zero). Subtracting the common constant $U_n$ from both sides yields
\[
\mathbb{E}[M_{n,k}] - U_n = \mathbb{E}[M_{n,m}] - U_n.
\]
By Definition~\ref{def:Delta_k}, the left-hand side is $\Delta_{n,k}^{\text{mem}}$ and the right-hand side is $\Delta_{n,m}^{\text{mem}}$. Therefore,
\[
\Delta_{n,k}^{\text{mem}} = \Delta_{n,m}^{\text{mem}},
\]
which proves \eqref{eq:saturation_equality}. Since $k \ge m$ was arbitrary, the equality holds for all $m \le k \le n$, completing the proof.
\end{proof}

\begin{remark}[Interpretation of the saturation property]
\label{rem:saturation_interpretation}
Theorem~\ref{thm:Delta_k_saturation} demonstrates that for an $m$-th order Markov chain, the FH-MPE saturates once the memory length reaches the order $m$, and remains constant for all larger memory lengths up to the full history length $n$. This is consistent with the Markov property: the recent $m$ observations capture all predictive information contained in the entire history, so extending the memory beyond $m$ steps cannot increase the FH-MPE.

Together with the monotonicity result in Theorem~\ref{thm:Delta_k_monotonicity}, the saturation property provides a complete picture of how the FH-MPE evolves with memory length for Markov chains. The monotonicity theorem ensures that $\Delta_{n,k}^{\text{mem}}$ is nondecreasing in $k$, while the saturation theorem identifies where this increase stops. Combining them yields the chain stated in Corollary~\ref{cor:saturation_chain}:
\[
0 = \Delta_{n,0}^{\text{mem}} \le \cdots \le \Delta_{n,m}^{\text{mem}} = \Delta_{n,m+1}^{\text{mem}} = \cdots = \Delta_{n,n}^{\text{mem}} = \Delta_n^{\text{mem}},
\]
in which monotonicity accounts for the inequalities and saturation accounts for the equalities from $m$ onward.

In addition, it offers a possible quantitative perspective for comparing the memory requirements of different processes, by indicating the length of history beyond which no further predictive gain is attainable. For the case of finite-order Markov chains, this perspective is made precise in Theorem~\ref{thm:effective_memory_markov}, which establishes the explicit upper bound \(L_n \le m\) (see Definition~\ref{def:effective_memory_length}) on the minimal memory length required to attain the full-history MPE.
\end{remark}

The following corollary combines the saturation theorem with the monotonicity chain established in Corollary~\ref{cor:Delta_k_chain} to obtain the complete evolution of the FH-MPE for $m$-th order Markov chains.

\begin{corollary}[Saturation and the monotonicity chain for Markov chains]\label{cor:saturation_chain}
Under the assumptions of Theorem~\ref{thm:Delta_k_saturation}, for every $n \ge m$,
\begin{equation}
0 = \Delta_{n,0}^{\text{mem}} \le \Delta_{n,1}^{\text{mem}} \le \cdots \le \Delta_{n,m-1}^{\text{mem}} \le \Delta_{n,m}^{\text{mem}} = \Delta_{n,m+1}^{\text{mem}} = \cdots = \Delta_{n,n}^{\text{mem}} = \Delta_n^{\text{mem}}.
\label{eq:saturation_chain}
\end{equation}
Thus the FH-MPE increases monotonically from $0$ up to $\Delta_{n,m}^{\text{mem}}$, and then remains constant for all $k \ge m$.
\end{corollary}

\begin{proof}
The monotonicity chain from Corollary~\ref{cor:Delta_k_chain} gives
\[
0 = \Delta_{n,0}^{\text{mem}} \le \Delta_{n,1}^{\text{mem}} \le \cdots \le \Delta_{n,m-1}^{\text{mem}} \le \Delta_{n,m}^{\text{mem}} \le \cdots \le \Delta_{n,n}^{\text{mem}} = \Delta_n^{\text{mem}}.
\]
Theorem~\ref{thm:Delta_k_saturation} establishes that $\Delta_{n,k}^{\text{mem}} = \Delta_{n,m}^{\text{mem}}$ for all $m \le k \le n$. In particular, taking $k=n$ gives $\Delta_{n,n}^{\text{mem}} = \Delta_{n,m}^{\text{mem}}$. By \eqref{eq:Delta_k_boundary_full}, we have $\Delta_{n,n}^{\text{mem}} = \Delta_n^{\text{mem}}$, and therefore $\Delta_{n,m}^{\text{mem}} = \Delta_n^{\text{mem}}$. Applying these equalities—both the saturation equalities for $m+1, \dots, n$ and the identification $\Delta_{n,m}^{\text{mem}} = \Delta_n^{\text{mem}}$—to the chain above yields
\[
0 = \Delta_{n,0}^{\text{mem}} \le \Delta_{n,1}^{\text{mem}} \le \cdots \le \Delta_{n,m-1}^{\text{mem}} \le \Delta_{n,m}^{\text{mem}} = \Delta_{n,m+1}^{\text{mem}} = \cdots = \Delta_{n,n}^{\text{mem}} = \Delta_n^{\text{mem}},
\]
which is precisely \eqref{eq:saturation_chain}.
\end{proof}

\subsection{Minimal Memory Length for Full-History MPE}
\label{subsec:effective_memory}

The developments in the previous sections together offer a coherent quantitative account of how prediction accuracy improves with memory length in the discrete-time finite-state setting. In Section~\ref{sec:definition}, we introduced the Memory Prediction Excess $\Delta_n^{\text{mem}}$ (Definition~\ref{def:MPE}), which measures the average improvement in optimal prediction accuracy achieved by conditioning on the full history $X_1^n$, with the static marginal distribution $U_n := \max_{a \in \mathcal{A}} \mathbb{P}(X_{n+1}=a)$ (see \eqref{eq:U_n_def}) serving as the baseline. The fundamental properties of this quantity were subsequently established in Section~\ref{sec:properties}: Theorem~\ref{thm:nonneg} confirmed its non-negativity, Proposition~\ref{prop:upper_bound} provided a sharp upper bound in terms of the static accuracy, and Theorem~\ref{thm:sharpness} characterized the precise condition under which this bound is attained—namely, when the future is almost surely a deterministic function of the past. Two degenerate scenarios in which the MPE vanishes were also treated in Corollaries~\ref{cor:independence} and~\ref{cor:static_degenerate}: independence of the future from the past, and the case where the static prediction is already perfect ($U_n = 1$). 

In Section~\ref{sec:normalized}, we introduced the normalized Memory Prediction Excess $\Delta_n^{\text{eff}}$ (Definition~\ref{def:eff}), which rescales $\Delta_n^{\text{mem}}$ by the maximal possible excess $1 - U_n$, thereby providing a dimensionless measure of predictive efficiency that facilitates comparisons across different processes and time indices. Proposition~\ref{prop:eff_zero_criteria} characterized the scenarios in which this normalized quantity vanishes, while Proposition~\ref{prop:eff_one_criteria} identified the condition $\Delta_n^{\text{eff}} = 1$ with the deterministic dependence of $X_{n+1}$ on $X_1^n$. Additionally, Proposition~\ref{prop:MPE_lower_bound} established a lower bound for $\Delta_n^{\text{mem}}$ by comparing predictions based on histories of different lengths, showing that the expected optimal prediction accuracy is monotone with respect to the history length.

In Section~\ref{sec:finite_history}, we extended this framework to histories of finite length, introducing the family $\{\Delta_{n,k}^{\text{mem}}\}_{0 \le k \le n}$ (Definition~\ref{def:Delta_k}), where $k$ denotes the number of most recent observations retained. The monotonicity of this family with respect to $k$ was established in Theorem~\ref{thm:Delta_k_monotonicity}: as the memory length increases, the FH-MPE does not decrease. This led to the chain of inequalities recorded in Corollary~\ref{cor:Delta_k_chain}:
\[
0 = \Delta_{n,0}^{\text{mem}} \le \Delta_{n,1}^{\text{mem}} \le \cdots \le \Delta_{n,n}^{\text{mem}} = \Delta_n^{\text{mem}},
\]
where the final equality $\Delta_{n,n}^{\text{mem}} = \Delta_n^{\text{mem}}$ follows from the boundary identity \eqref{eq:Delta_k_boundary_full}. For the important class of finite-order Markov chains, we further showed in Theorem~\ref{thm:Delta_k_saturation} that this monotone increase saturates once the memory length reaches the Markov order; for $k \ge m$, the FH-MPE remains constant at $\Delta_{n,m}^{\text{mem}} = \Delta_n^{\text{mem}}$ (see also Remark~\ref{rem:saturation_interpretation} for a discussion of the distinction between the Markov order and the saturation point of the FH-MPE).

The monotonicity chain above shows that the FH-MPE increases monotonically from $\Delta_{n,0}^{\text{mem}} = 0$ at $k=0$ to $\Delta_{n,n}^{\text{mem}} = \Delta_n^{\text{mem}}$ at $k=n$. This observation naturally leads to the following question: what is the minimal memory length at which the FH-MPE already attains the value of the full-history MPE? As foreshadowed in Remark~\ref{rem:saturation_interpretation}, this question motivates the notion of minimal memory length for full-history MPE (see Definition~\ref{def:effective_memory_length} below), which we now formalize.

\begin{definition}[Minimal Memory Length for Full-History MPE]\label{def:effective_memory_length}
For a fixed time index $n \ge 1$, the \textbf{minimal memory length for full-history MPE} of the process $\{X_n\}_{n\ge 1}$ is defined as
\begin{equation}
L_n := \min \left\{ k \in \{0,1,\dots,n\} : \Delta_{n,k}^{\text{mem}} = \Delta_n^{\text{mem}} \right\},
\label{eq:L_n_def}
\end{equation}
where $\Delta_{n,k}^{\text{mem}}$ denotes the FH-MPE with memory length $k$ (Definition~\ref{def:Delta_k}) and $\Delta_n^{\text{mem}}$ is the full-history MPE (Definition~\ref{def:MPE}).
\end{definition}

Before interpreting the meaning of $L_n$, we first verify that the definition is legitimate—that is, the minimum exists. The set appearing on the right-hand side of \eqref{eq:L_n_def} is nonempty: by the boundary identity \eqref{eq:Delta_k_boundary_full}, we have $\Delta_{n,n}^{\text{mem}} = \Delta_n^{\text{mem}}$, so that $k = n$ is always an element of the set. Since the set is a finite nonempty subset of the nonnegative integers, it possesses a least element (by the well-ordering property). Hence the minimum in \eqref{eq:L_n_def} exists and is uniquely defined.

Thus $L_n$ is the smallest memory length for which the FH-MPE coincides with the full-history MPE. From the definition and the preceding existence argument, we obtain directly $0 \le L_n \le n$ and, by definition, $\Delta_{n,L_n}^{\text{mem}} = \Delta_n^{\text{mem}}$.

The following proposition records the elementary properties of $L_n$.

\begin{proposition}[Basic properties of the minimal memory length for full-history MPE]\label{prop:effective_memory_basic}
For every $n \ge 1$, the following hold.
\begin{enumerate}
    \item The minimal memory length for full-history MPE satisfies
    \begin{equation}
    0 \le L_n \le n, \qquad \Delta_{n,L_n}^{\text{mem}} = \Delta_n^{\text{mem}}.
    \label{eq:L_n_basic}
    \end{equation}
    \item For every integer $k$ with $L_n \le k \le n$,
    \begin{equation}
    \Delta_{n,k}^{\text{mem}} = \Delta_n^{\text{mem}}.
    \label{eq:L_n_saturation}
    \end{equation}
    In words, once the memory length reaches $L_n$, the FH-MPE remains constant at the full-history value for all larger memory lengths.
    \item If $U_n = 1$, then $L_n = 0$.
\end{enumerate}
\end{proposition}

\begin{proof}
We prove the three assertions in order.

\paragraph{Proof of (i).} 
We verify the two claims separately.

First, we establish the bounds on $L_n$. Let $S_n$ denote the set appearing in the definition of $L_n$ in \eqref{eq:L_n_def}, i.e.,
\[
S_n := \{ k \in \{0,1,\dots,n\} : \Delta_{n,k}^{\text{mem}} = \Delta_n^{\text{mem}} \}.
\]
Since $S_n \subseteq \{0,1,\dots,n\}$, every element $k \in S_n$ satisfies $0 \le k \le n$. Consequently, the minimum of $S_n$, namely $L_n$, also satisfies $0 \le L_n \le n$.

Second, we verify the equality $\Delta_{n,L_n}^{\text{mem}} = \Delta_n^{\text{mem}}$. By the definition of $L_n$ as the minimum of $S_n$, we have $L_n \in S_n$. The membership condition for $S_n$ is precisely $\Delta_{n,k}^{\text{mem}} = \Delta_n^{\text{mem}}$. Substituting $k = L_n$ gives $\Delta_{n,L_n}^{\text{mem}} = \Delta_n^{\text{mem}}$. Thus both claims in (i) are established.

\paragraph{Proof of (ii).} 
Fix an arbitrary integer $k$ satisfying $L_n \le k \le n$. We prove the equality $\Delta_{n,k}^{\text{mem}} = \Delta_n^{\text{mem}}$ by establishing two inequalities and then combining them.

\textit{Step 1: Proving $\Delta_{n,k}^{\text{mem}} \ge \Delta_n^{\text{mem}}$.} 
By the monotonicity chain in Corollary~\ref{cor:Delta_k_chain}, for any integers $a$ and $b$ with $0 \le a \le b \le n$, we have $\Delta_{n,a}^{\text{mem}} \le \Delta_{n,b}^{\text{mem}}$. Since $L_n \le k$ and both indices lie in the admissible range, we may apply this monotonicity with $a = L_n$ and $b = k$ to obtain
\begin{equation}
\Delta_{n,L_n}^{\text{mem}} \le \Delta_{n,k}^{\text{mem}}.
\label{eq:monotonicity_Ln_to_k}
\end{equation}
By property (i) established above (specifically, $\Delta_{n,L_n}^{\text{mem}} = \Delta_n^{\text{mem}}$ from \eqref{eq:L_n_basic}), substituting this identity into \eqref{eq:monotonicity_Ln_to_k} gives
\[
\Delta_n^{\text{mem}} \le \Delta_{n,k}^{\text{mem}}.
\]

\textit{Step 2: Proving $\Delta_{n,k}^{\text{mem}} \le \Delta_n^{\text{mem}}$.} 
Again by the monotonicity chain in Corollary~\ref{cor:Delta_k_chain}, since $k \le n$ and both indices lie in the admissible range, we may apply it with $a = k$ and $b = n$ to obtain
\begin{equation}
\Delta_{n,k}^{\text{mem}} \le \Delta_{n,n}^{\text{mem}}.
\label{eq:monotonicity_k_to_n}
\end{equation}
By the boundary identity \eqref{eq:Delta_k_boundary_full}, we have $\Delta_{n,n}^{\text{mem}} = \Delta_n^{\text{mem}}$. Substituting this into \eqref{eq:monotonicity_k_to_n} gives
\[
\Delta_{n,k}^{\text{mem}} \le \Delta_n^{\text{mem}}.
\]

\textit{Step 3: Combining the two inequalities.} 
Combining the inequalities obtained in Steps 1 and 2, together with property (i) and the boundary identity \eqref{eq:Delta_k_boundary_full}, we have
\[
\Delta_n^{\text{mem}} \le \Delta_{n,k}^{\text{mem}} \le \Delta_n^{\text{mem}}.
\]
The left and right sides are identical; hence both inequalities are equalities, and therefore
\[
\Delta_{n,k}^{\text{mem}} = \Delta_n^{\text{mem}}.
\]
Since $k$ was arbitrary in the range $L_n \le k \le n$, the assertion in (ii) holds for all such $k$.

\paragraph{Proof of (iii).} 
Assume $U_n = 1$. We prove $L_n = 0$ by showing that $k=0$ belongs to the defining set of $L_n$, and then using the minimality of $L_n$.

\textit{Step 1: Establishing $\Delta_n^{\text{mem}} = 0$.} 
By Corollary~\ref{cor:static_degenerate}, the condition $U_n = 1$ implies that the full-history Memory Prediction Excess vanishes:
\[
\Delta_n^{\text{mem}} = 0.
\]

\textit{Step 2: Establishing $\Delta_{n,k}^{\text{mem}} = 0$ for all $k$.} 
For any $k$ with $0 \le k \le n$, the monotonicity chain in Corollary~\ref{cor:Delta_k_chain} gives
\begin{equation}
0 \le \Delta_{n,k}^{\text{mem}} \le \Delta_{n,n}^{\text{mem}}.
\label{eq:monotonicity_bound_for_all_k}
\end{equation}
By the boundary identity \eqref{eq:Delta_k_boundary_full}, we have $\Delta_{n,n}^{\text{mem}} = \Delta_n^{\text{mem}}$. Combining this with Step 1 yields
\[
\Delta_{n,n}^{\text{mem}} = 0.
\]
Substituting this into \eqref{eq:monotonicity_bound_for_all_k} yields, for every $0 \le k \le n$,
\[
0 \le \Delta_{n,k}^{\text{mem}} \le 0,
\]
which implies
\[
\Delta_{n,k}^{\text{mem}} = 0.
\]

\textit{Step 3: Identifying $k=0$ as an element of the defining set.} 
Taking $k=0$ in the conclusion of Step 2 gives
\[
\Delta_{n,0}^{\text{mem}} = 0.
\]
From \textit{Step 1} of this proof, we also have $\Delta_n^{\text{mem}} = 0$. Hence $\Delta_{n,0}^{\text{mem}} = \Delta_n^{\text{mem}}$. By the definition of $S_n$ introduced in the proof of (i) above, this means that $0 \in S_n$.

\textit{Step 4: Concluding $L_n = 0$.} 
Since $L_n$ is the minimum of $S_n$, and $0$ is the smallest integer in $\{0,1,\dots,n\}$, the membership $0 \in S_n$ implies that the minimum of $S_n$ cannot exceed $0$; that is, $L_n \le 0$. On the other hand, $L_n \ge 0$ follows from property (i) (specifically, $0 \le L_n \le n$ established in \eqref{eq:L_n_basic}). Therefore $L_n = 0$.
\end{proof}

Proposition~\ref{prop:effective_memory_basic} collects the basic properties of $L_n$ that follow from its definition (Definition~\ref{def:effective_memory_length}) in conjunction with the monotonicity chain (Corollary~\ref{cor:Delta_k_chain}) and the boundary identities (Definition~\ref{def:Delta_k}), together with the static degeneracy result (Corollary~\ref{cor:static_degenerate}). These properties hold across processes in the discrete-time finite-state setting considered throughout this work, and they provide a foundational characterization of $L_n$ in the absence of additional structural assumptions.

For processes with additional structural assumptions, sharper characterizations of $L_n$ can be derived. In particular, for finite-order Markov chains, the saturation theorem (Theorem~\ref{thm:Delta_k_saturation}) yields an explicit upper bound on $L_n$ in terms of the Markov order $m$, independent of $n$. We record this consequence in the following theorem.

\begin{theorem}[Upper bound for the minimal memory length for full-history MPE in Markov chains]\label{thm:effective_memory_markov}
Let $\{X_n\}_{n\ge 1}$ be an $m$-th order Markov chain (see Definition~\ref{def:m_markov}) for some fixed integer $m \ge 1$. Then for every $n \ge m$,
\begin{equation}
L_n \le m.
\label{eq:L_n_markov_bound}
\end{equation}
\end{theorem}

\begin{proof}
We prove the bound by showing that the specific value $k = m$ satisfies the defining condition of $L_n$, namely $\Delta_{n,m}^{\text{mem}} = \Delta_n^{\text{mem}}$. Once this equality is established, the definition of $L_n$ as the minimum of the set of memory lengths achieving full-history performance immediately yields $L_n \le m$. The proof is organized into six steps.

\paragraph{Step 1: Recalling the saturation theorem for the Markov chain.}
Since the process is assumed to be an $m$-th order Markov chain and $n \ge m$, Theorem~\ref{thm:Delta_k_saturation} applies directly. That theorem establishes that, for such processes, once the memory length reaches the Markov order $m$, the FH-MPE saturates and remains constant for all larger memory lengths. Specifically, for every integer $k$ with $m \le k \le n$,
\begin{equation}
\Delta_{n,k}^{\text{mem}} = \Delta_{n,m}^{\text{mem}}.
\label{eq:effective_proof_saturation}
\end{equation}
This equality is the key saturation property derived in Theorem~\ref{thm:Delta_k_saturation}; it reflects the fact that for an $m$-th order Markov chain, the additional older observations beyond the most recent $m$ steps do not alter the conditional distribution of $X_{n+1}$.

\paragraph{Step 2: Applying the saturation equality at the full-history length.}
Since $n \ge m$, the full-history length $k = n$ satisfies the condition $m \le k \le n$ required for the application of \eqref{eq:effective_proof_saturation}. Substituting $k = n$ into \eqref{eq:effective_proof_saturation} yields
\begin{equation}
\Delta_{n,n}^{\text{mem}} = \Delta_{n,m}^{\text{mem}}.
\label{eq:effective_proof_k_n}
\end{equation}
This step shows that the FH-MPE based on the full history $X_1^n$—which is the complete set of available observations from time $1$ to $n$—coincides with the FH-MPE based on only the most recent $m$ observations.

\paragraph{Step 3: Identifying the full-history FH-MPE with the original MPE.}
By the boundary identity \eqref{eq:Delta_k_boundary_full}, when the memory length $k$ equals the full history length $n$, the finite-history MPE reduces to the original full-history MPE:
\begin{equation}
\Delta_{n,n}^{\text{mem}} = \Delta_n^{\text{mem}}.
\label{eq:effective_proof_boundary}
\end{equation}
This identity is a direct consequence of Definition~\ref{def:Delta_k}: when $k = n$, the conditioning block $X_{n-k+1}^n$ becomes $X_1^n$, which is precisely the full history used in the definition of $\Delta_n^{\text{mem}}$. Thus $\Delta_{n,n}^{\text{mem}}$ and $\Delta_n^{\text{mem}}$ are not merely equal in value; they represent the same probabilistic quantity evaluated under identical conditioning information.

\paragraph{Step 4: Combining the preceding equalities.}
Combining \eqref{eq:effective_proof_boundary} with \eqref{eq:effective_proof_k_n}—by replacing $\Delta_{n,n}^{\text{mem}}$ in the latter with $\Delta_n^{\text{mem}}$ from the former—we eliminate $\Delta_{n,n}^{\text{mem}}$ and obtain the relation between $\Delta_{n,m}^{\text{mem}}$ and the full-history MPE:
\begin{equation}
\Delta_n^{\text{mem}} = \Delta_{n,m}^{\text{mem}}.
\label{eq:effective_proof_target}
\end{equation}
Equivalently, we have $\Delta_{n,m}^{\text{mem}} = \Delta_n^{\text{mem}}$. This is the central equality of the proof: it establishes that a memory of length $m$—the Markov order—already attains the predictive accuracy achievable with the full history $X_1^n$.

\paragraph{Step 5: Verifying membership of $m$ in the defining set.}
Let $S_n$ denote the set introduced in the proof of Proposition~\ref{prop:effective_memory_basic}, i.e.,
\[
S_n := \left\{ k \in \{0,1,\dots,n\} : \Delta_{n,k}^{\text{mem}} = \Delta_n^{\text{mem}} \right\}.
\]
The set $S_n$ collects precisely those memory lengths $k$ for which the FH-MPE attains the full-history value. By \eqref{eq:effective_proof_target}, we have $\Delta_{n,m}^{\text{mem}} = \Delta_n^{\text{mem}}$; therefore $m$ satisfies the membership criterion and hence $m \in S_n$. This identifies an explicit element of $S_n$, which will be used to obtain the desired bound on the minimum.

\paragraph{Step 6: Concluding the bound on $L_n$.}
Since $L_n$ is defined as $\min S_n$ and $S_n$ is a nonempty subset of $\{0,1,\dots,n\}$ (as established in Definition~\ref{def:effective_memory_length}), the minimum is well-defined. With $m \in S_n$ established in Step 5, the minimum of $S_n$ cannot exceed any of its elements; in particular, it cannot exceed $m$. Therefore,
\[
L_n = \min S_n \le m.
\]
This completes the proof.
\end{proof}

\begin{remark}[Significance of the minimal memory length for full-history MPE]
\label{rem:effective_memory_length}
The minimal memory length $L_n$ introduced in Definition~\ref{def:effective_memory_length} provides a quantitative answer to the question of how much of the past is actually needed to achieve the optimal prediction accuracy attainable with the full history. While the Markov order $m$ characterizes the full conditional dependence structure of the process, $L_n$ captures the same predictive sufficiency from a different perspective: it is the shortest history length that already attains the full-history MPE $\Delta_n^{\text{mem}}$. For an $m$-th order Markov chain, Theorem~\ref{thm:effective_memory_markov} gives the upper bound $L_n \le m$. Once the memory length reaches $L_n$, Proposition~\ref{prop:effective_memory_basic} (ii) ensures that the FH-MPE remains constant at $\Delta_n^{\text{mem}}$ for all larger memory lengths. Thus $L_n$ identifies the threshold beyond which additional past observations do not improve the value of the FH-MPE $\Delta_{n,k}^{\text{mem}}$.
\end{remark}



\section*{ACKNOWLEDGMENTS}
The author would like to express sincere gratitude to arXiv for providing an open and accessible platform for the dissemination of research. Thanks are also due to the mathematical community for its continued dedication to open science, which makes works like this one possible.

\section*{Data Availability Statement}
Data sharing is not applicable to this article as no datasets were generated or analysed during the current study.




\end{document}